\documentclass[letterpaper,11pt]{article}
\usepackage{typearea}
\typearea{15}

\usepackage{setspace}
\usepackage{latexsym,graphicx}
\usepackage{amsthm,amsmath,amssymb,enumerate}
\usepackage{empheq}
\usepackage{xspace}
\usepackage{bm}
\usepackage{ifpdf}
\usepackage{color}
\usepackage{algorithm}
\usepackage[noend]{algpseudocode}
\usepackage{nicefrac}
\usepackage{wrapfig}

\allowdisplaybreaks

\definecolor{Darkblue}{rgb}{0,0,0.4}
\definecolor{Brown}{cmyk}{0,0.81,1.,0.60}
\definecolor{Purple}{cmyk}{0.45,0.86,0,0}
\usepackage{hyperref}
\hypersetup{colorlinks=true,
           citebordercolor={.6 .6 .6},linkbordercolor={.6 .6 .6},%
citecolor=Darkblue,urlcolor=black,linkcolor=red,pagecolor=black}
\newcommand{\lref}[2][]{\hyperref[#2]{#1~\ref*{#2}}}


\newtheorem{theorem}{Theorem}[section]

\newtheorem{proposition}[theorem]{Proposition}

\newtheorem{lemma}[theorem]{Lemma}

\newtheorem{claim}[theorem]{Claim}

\numberwithin{algorithm}{section}

\newenvironment{aside}{

\begin{quote}
  \footnotesize
}{\end{quote}

}

\newcommand{\junk}[1]{}
\newcommand{\ignore}[1]{}

\newcommand{\R}[0]{{\ensuremath{\mathbb{R}}}}

\newcommand{\argmin}{\operatorname{argmin}}

\newcommand{\sse}{\subseteq}

\newcommand{\e}{\varepsilon}
\newcommand{\eps}{\varepsilon}
\newcommand{\ts}{\textstyle}

\newcounter{note}[section]
\renewcommand{\thenote}{\thesection.\arabic{note}}
\newcommand{\agnote}[1]{\refstepcounter{note}$\ll${\bf Anupam~\thenote:}
  {\sf \color{blue} #1}$\gg$\marginpar{\tiny\bf AG~\thenote}}

\newcommand{\alert}[1]{{\color{red}#1}}

\newcommand{\balert}[1]{{\color{blue}#1}}

\newcommand{\qedsymb}{\hfill{\rule{2mm}{2mm}}}

\newcommand{\initOneLiners}{%
    \setlength{\itemsep}{0pt}
    \setlength{\parsep }{0pt}
    \setlength{\topsep }{0pt}
}

\newcommand{\squishlist}{
 \begin{list}{$\bullet$}
  { \setlength{\itemsep}{0pt}
     \setlength{\parsep}{3pt}
     \setlength{\topsep}{3pt}
     \setlength{\partopsep}{0pt}
     \setlength{\leftmargin}{1.5em}
     \setlength{\labelwidth}{1em}
     \setlength{\labelsep}{0.5em} } }

\newcommand{\squishend}{
  \end{list}  }

\newcommand{\gr}{\nabla}
\newcommand{\ip}[1]{\langle #1 \rangle}
\newcommand{\braket}[2]{\langle #1, #2 \rangle}

\DeclarePairedDelimiterX{\infdivx}[2]{(}{)}{%
  #1\;\delimsize\|\;#2%
}

\newcommand{\hdiv}{D_h\infdivx}
\newcommand{\KLdiv}{KL\infdivx}
\DeclarePairedDelimiter{\norm}{\lVert}{\rVert}

\newcommand*\widefbox[1]{\fbox{\hspace{2em}#1\hspace{2em}}}

\begin{document}

\title{{\bf Potential-Function Proofs for First-Order
    Methods}\thanks{This work was done in part while the authors were
    visiting the \emph{Algorithms and Uncertainty} and \emph{Bridging
      Discrete and Continuous Optimization} programs at the Simons
    Institute for the Theory of Computing. It was partially supported by
    the DIMACS/Simons Collaboration on Bridging Continuous and Discrete
    Optimization through NSF grant \#CCF-1740425.} }

\author{Nikhil Bansal\thanks{Department of Mathematics and Computer
    Science, Eindhoven University of Technology, Netherlands. Supported
    in part by NWO Vidi grant 639.022.211 and ERC consolidator grant
    617951.} \and Anupam Gupta\thanks{Computer Science Department,
    Carnegie Mellon University, Pittsburgh, PA 15213. Supported in part
    by NSF awards CCF-1536002, CCF-1540541, and CCF-1617790.}}

\maketitle

\begin{abstract}
  This technical note discusses proofs of convergence for first-order
  methods based on simple potential-function arguments. We cover methods
  like gradient descent (for both smooth and non-smooth settings),
  mirror descent, and some accelerated variants. We hope the structure
  and presentation of these amortized-analysis proofs will be useful as
  a guiding principle in  learning and using these proofs.
\end{abstract}

\section{Introduction}

The so-called ``gradient descent'' framework is a class of iterative
methods for solving convex minimization problems --- indeed, since the
gradient gives the direction of steepest increase in function value, a
natural approach to minimize the convex function is to move in the
direction opposite to the gradient. Variants of this general versatile
approach have been central to convex optimization for many years. In
recent years, with the increased use of continuous methods in discrete
optimization, and with the gap between continuous and discrete
optimization being smaller than ever, this technique has also become
central for algorithm design in general.

In this note we give convergence arguments for many commonly studied
versions of first-order methods using simple \emph{potential-function}
arguments.  We find that presenting the proofs in the amortized-analysis
framework is useful as a guiding principle, since it imparts a clear
structure and direction to proofs. We hope others will also find this
perspective useful, both in learning and teaching these techniques and
proofs, and also in extending them to other domains.

A disclaimer: previously existing proofs for gradient methods are usually not 
difficult, and their individual components are not substantially different from the 
ones in this note. However, using an explicit potential to guide our
proofs makes them arguably more intuitive. In fact, the intuition of 
viewing these gradient methods as trying to control a potential function
is also known to the specialists; e.g., see the text of Nemirovski and
Yudin~\cite[pp.~85--88]{NY} for a continuous perspective via Lyapunov
functions. This is more explicit in recent
papers~\cite{SBC,Wibi,KBB15,WilsonRJ16,DO17} relating continuous and
discrete updates to understand the acceleration phoenomenon. E.g.,
Krichene et al.~\cite{KBB15} give the potential function we use in
\S\ref{sec:nest-proof}. However, these potential function proofs
and intuitions have not yet permeated into the commonly presented
expositions. The current note is an attempt to make such ideas more widely known.


\paragraph{Basic Definitions.} 
Recall that a set $K \sse \R^d$ is \emph{convex} if for all $x,y \in K$,
the \emph{convex combination} $\lambda x + (1-\lambda)y \in K$ for all
$\lambda \in [0,1]$.  A function $f: \R^d \to R$  is \emph{convex} over a
convex set $K$ if
\[ f(\lambda x + (1-\lambda)y) \leq \lambda\, f(x) + (1-\lambda)\, f(y)
  \qquad \forall x, y \in K, \forall \lambda \in [0,1]. \] This is
called the \emph{zeroth-order} definition. There are other equivalent
notions: if the function is differentiable, the \emph{first-order}
definition is that $f$ is convex over $K$ if
\begin{gather}
  f(y) \geq f(x) + \ip{ \gr f(x), y-x} \qquad \forall x,y \in K. \label{eq:conv-def}
\end{gather}
(The \emph{second-order} definition says that a twice-differentiable $f$
is convex if its Hessian matrix $\gr^2 f$ is positive-semidefinite.)
For this note, we assume our convex sets $K$ are closed, and the convex functions $f$ are
differentiable. 
However, the proofs extend to non-differentiable
functions in the natural way, using subgradients. (See, e.g.,~\cite{HUL}
for more definitions and background on convexity and subgradients.) 

\paragraph{The Problems.} Given a convex function $f: \R^d \to R$, and an error parameter $\e$, the
\emph{(unconstrained) convex minimization} problem is to find a point
$\widehat{x}$ such that
$f(\widehat{x}) - \min_{x \in \R^d} f(x) \leq \e$.  In the
\emph{constrained} version of the problem, we are also given a convex
set $K$, and the goal is to find a point
$\widehat{x} \in K$ which has \emph{error}
$f(\widehat{x}) - \min_{x \in K} f(x) \leq \e$. In either case, let
$x^*$ denote the minimizer for $f(\cdot)$. We will be interested in bounding the number 
of gradient queries required to converge to the approximate minimizer $\widehat{x}$, as function of the distance between $x_0$ and $x^*$ and some parameters of the function $f$, depending on the particular variant of gradient descent.

In \emph{online convex optimization} over a convex set $K$, at each
timestep $t = 1,2, \cdots$, the algorithm outputs a point $x_t \in K$
and an adversary produces a convex function $f_t$. The algorithm's
\emph{loss} at timestep $t$ is defined to be $f_t(x_t)$. Now the
\emph{regret} of the algorithm is
$\sum_{t = 1}^T f_t(x_t) - \min_{x \in K} \sum_{t=1}^T f_t(x)$, and the
goal is to determine the points $x_t$ online (without the knowledge of
the current or future functions $\{f_s\}_{s \geq t}$) to minimize the
regret.  Note that this generalizes the convex optimization setting
above, which corresponds to setting $f_t=f$ at each time $t$, and that
any algorithm with sublinear regret $o(T)$ can be used to find an
approximate optimum $\widehat{x}$ up to any desired accuracy $\e$.

\paragraph{Assumptions:} We assume that our convex functions are closed,
convex, and differentiable, and that the convex sets $K$ are also closed
with non-empty interior.  We assume access to a gradient oracle: i.e.,
given any point $x$, we can get the gradient $\gr f(x)$ of the function
$f$ at any point $x$. We only work with the Euclidean norm
$\norm{\cdot}_2$ for the first few sections; general norms are discussed
in \S\ref{sec:MD}.

\paragraph{References:}
In this technical survey, we focus only on the exposition of the proofs.
We omit most citations, and also discussion of the ``bigger picture''.

There are many excellent sources for other proofs of these results, with
comprehensive bibliographies; e.g., see the authoritative notes by
Nesterov~\cite{Nest-book}, Ben-Tal and Nemirovski~\cite{BTN}, the
monographs of Bubeck~\cite{Seb} and Shalev-Shwartz~\cite{SSS}, the
textbooks by Cesa-Bianchi and Lugosi~\cite{CBL}, Hazan~\cite{Hazan}, and
lecture notes by Duchi~\cite{Duchi} and Vishnoi~\cite{Vishnoi}. 

There are several other perspectives on these proofs that the reader may
find useful. One useful perspective is that of viewing first-order
methods as discretizations of suitable continuous dynamics; this appears
even in the classic work of Nemirovski and Yudin~\cite{NY}, and has been
widely used recently (see,
e.g.,~\cite{SBC,Wibi,KBB15,WilsonRJ16}). Another useful perspective is
exhibit a ``dual'' lower bound on the optimal value via the convex
conjugate, and use the duality gap to bound the error (see,
e.g.,~\cite{DO17,Pena17}). We point the interested reader to the
respective papers for more details.

Finally, we discuss some concurrent and related work. Independently of
our work, Karimi and Vavasis~\cite{KarimiV17} give potential-based convergence
proofs for conjugate gradient and accelerated methods; their potentials
are similar to ours. And following up on a preprint of our results,
Taylor and Bach~\cite{TB19} analyze stochastic first-order methods using
potential functions.

\subsection{Results and Organization}

All of the proofs use the same general potential: for some fixed point
$x^*$ (which can be thought of as the optimal or reference point) we have
\begin{gather}
  \Phi_t = a_t\cdot (f(x_t) - f(x^*)) + b_t \cdot (\text{distance from
    $x_t$ to $x^*$}). \label{eq:meta}
\end{gather}
Here $a_t, b_t$ are non-negative, and naturally, different proofs use slightly different choice of these multipliers, and even the distance functions may vary. However, the
general approach remains the same: we show that
$\Phi_{t+1} - \Phi_t \leq B_t$ (where $B_t$ is often zero). Since the
potential and distance terms remain non-negative, the telescoping sum
gives
\[ \Phi_T \leq \Phi_0 + \sum_{t =0}^{T-1} B_t \quad\implies\quad f(x_T) - f(x^*) \leq \frac{\Phi_0 +
    \sum_t B_t}{a_T}. \]

We begin in \S\ref{sec:online} with proofs of the basic
(projected) gradient descent, for general and strongly convex functions;
these even work in the online regret-minimization setting where the
function may change at each timestep. Here the
analysis is more along the lines of \emph{amortized-analysis}: we show
that the amortized cost, namely the cost of the algorithm plus the
increase in potential is at most the optimal cost (plus $B$). I.e.,
$f_t(x_t) + (\Phi_{t+1} - \Phi_t) \leq f_t(x^*) + B$. This telescopes to
imply that the average regret is
$\frac1T( \sum_{t = 1}^T (f_t(x_t) - f_t(x^*))) \leq B + \Phi_0/T$. The
potential here is very simple: we set $a_t = 0$ and just use the
distance of the current point $x_t$ to the optimal point $x^*$
(according to the ``right'' distance). E.g., for basic gradient descent,
the potential is just a scaled version of $\norm{x_t - x^*}^2$.

Next, we give proofs of convergence for the case of \emph{smooth} convex functions
in \S\ref{sec:smooth}. In the simplest case we just set $b_t = 0$
and use $a_t = t$ in~(\ref{eq:meta}) to prove $B_t \approx 1/t$. This gives an error of
$\approx (\log T)/T$, which is in the right ballpark. (This can be
optimized using better settings of the multipliers.)  The proofs for
projected smooth gradient descent, gradient descent for well-conditioned
functions, and the Frank-Wolfe method, all follow this template. For
these proofs, we now use the ``value-based'' terms in~(\ref{eq:meta}),
i.e., the terms that depend on $f(x_t) - f(x^*)$.

We then extend our understanding to \emph{mirror descent}. This is a
substantial generalization of gradient descent to general norms. While
the language necessarily becomes more technical (relying on dual norms
and Bregman divergences), the ideas remain clean. Indeed, the structure
of the potential-based proofs from \S\ref{sec:online} remains
essentially the same as for basic gradient descent; the potential is now
based on a Bregman diverence, a natural generalization of the squared
distance. These proofs appear in \S\ref{sec:MD}.

An orthogonal extension is to potential-based proofs of Nesterov's
accelerated gradient descent method for smooth and well-conditioned
convex functions.  The ideas in this section build on the simple
calculations we would have  seen in \S\ref{sec:online}
and \S\ref{sec:smooth}.  We can now use the full power of both
distance-based and value-based terms in the potential
function~(\ref{eq:meta}), trading them off against each other.
Moreover, in \S\ref{sec:nest-failed} we show how the basic
analysis for smooth convex functions from \S\ref{sec:smooth}
directly suggests how to obtain the accelerated algorithm by coupling
together one cautious and one aggressive gradient descent step.

\paragraph{Organization.} The paper follows the above outline. We start
with proofs for general convex functions in \S\ref{sec:online} using
simple distance-based potentials, then proceed to smooth and
well-conditioned convex functions in \S\ref{sec:smooth} using more
sophisticated potentials. We then discuss the generalization to mirror
descent via Bregman diverences in \S\ref{sec:MD}. Finally, we give
proofs for accelerated versions in \S\ref{sec:nesterov}. Note that
\S\ref{sec:MD} and \S\ref{sec:nesterov} are independent, and may be read
in any order.

\section{Online Analyses}
\label{sec:online}

\subsection{Basic Gradient Descent}
\label{sec:basic-GD}

The basic analysis works even for the online convex optimization case:
at each step we are given a function $f_t$, we play $x_t$, and want to
minimize the regret. In this case the update rule is:
\begin{gather}
  \boxed{x_{t+1} \gets x_{t} - \eta_t \, \gr f_t(x_t)}  \label{eq:gd-basic}
\end{gather}
An equivalent form for this update, that is easily verified by taking derivatives with respect to $x$, is:
\begin{gather}
  \boxed{x_{t+1} \gets \arg\min_x \Big\{ \frac12 \| x - x_{t} \|^2 +  \eta_t
    \ip{x, \gr f_t(x_t)} \Big\} }  \label{eq:gd-proxform}
\end{gather}
Intuitively, we want to move in the direction of the negative gradient,
but don't want to move too far. 

\begin{theorem}[Basic Gradient Descent]
  \label{thm:gd-basic-regret} Let $f_1,\ldots,f_T: \R^n \to \R$ be $G$-Lipschitz
  functions, i.e.~$\norm{\gr f_t(x)} \leq G$ for all
  $x, t$. Then starting at point $x_0 \in \R^n$
  and using updates~(\ref{eq:gd-basic}) with step size
  $\eta_t = \eta = \frac{D}{G\sqrt{T}}$ for $T$ steps guarantees an
  average regret of
  \[ \frac1T \sum_{t = 0}^{T-1} \big( f_t(x_t) - f_t(x^*) \big) \leq \eta
    \frac{G^2}{2} + \frac{D^2}{2\eta T} \leq \frac{DG}{\sqrt{T}},  \]
  for all $x^*$ with $ \norm{
    x_0 - x^*} \leq D$.

\end{theorem}

\begin{proof}
  Consider the potential function
  \begin{gather}
    \boxed{\Phi_t = \frac{1}{2\eta}\norm{x_t - x^*}^2}  \label{eq:pot}
  \end{gather}
  which is positive for all $t$.  We show that, for some upper
  bound $B$,
  \begin{equation}
    \label{eq:potbound}
    f_t(x_t) - f_t(x^*) + \Phi_{t+1} - \Phi_t \leq B.
  \end{equation}
  Summing over all times $t$, the average regret is
  \begin{align}
    \frac1T \sum_{t=0}^{T-1} (f_t(x_t) - f_t(x^*))
    &\leq B + \frac1T (\Phi_0 - \Phi_T) \leq B + \frac{\Phi_0}{T} = B + \frac{D^2}{2\eta
      T}. \label{eq:regret1} 
  \end{align}
  Now we can compute $B$, and then balance the two terms.  While the
  potential uses differences of the form $x_t - x^*$, the key is to
  express as much as possible in terms of $x_{t+1} - x_t$, because the
  update rule~(\ref{eq:gd-basic}) implies
  \begin{equation}
    \label{eq:f2}
    x_{t+1} - x_t = - \eta\gr f_t(x_t).
  \end{equation}

  \textbf{The Change in Potential.} Using that $\norm{a+b}^2 -
  \norm{a}^2 = 2\ip{a,b} + \norm{b}^2$
  for the Euclidean norm,
  \begin{align}
    \frac{1}{2}(\norm{x_{t+1} - x^*}^2 - \norm{x_t - x^*}^2) 
    &= \braket{x_{t+1} -
      x_t}{x_t - x^*} + \frac12 \norm{x_{t+1} - x_t}^2  \notag \\
    &= \eta_t \braket{\gr f_t(x_t)}{x^* - x_t} +
      \frac{\eta_t^2}{2}\norm{\gr f_t(x_t)}^2 \label{eq:potchange-basic}
  \end{align}
  \textbf{The Amortized Cost:} Setting  $\eta_t = \eta$ for all steps,
  \begin{align}
     f_t(x_t) &- f_t(x^*) + \Phi_{t+1} - \Phi_t \notag \\
    &= f_t(x_t) - f_t(x^*) + \braket{\gr f_t(x_t)}{x^* - x_t} +
      \frac{\eta}{2}\norm{\gr f_t(x_t)}^2 
      \tag{by~(\ref{eq:potchange-basic})} \\
    &\leq 0 + \frac{\eta}{2}\norm{\gr f_t(x_t)}^2
    \quad \leq \quad \frac{\eta G^2}{2}. \tag{by convexity, and the
      bound on gradients}
  \end{align}
  Substituting for $B$ in~(\ref{eq:regret1}) and simplifying with $\eta
  = \frac{D}{G\sqrt{T}}$, we get the theorem.
\end{proof}

The regret bound implies a convergence result for the offline case,
i.e., for the case where $f_t = f$ for all $t$. Here, setting
$\widehat{x} := \frac1T \sum_{t = 0}^{T-1} x_t$ shows
\begin{align*}
  f(\widehat{x}) - f(x^*) &= f\left(\frac1T \sum_t x_t \right) - f(x^*) 
                            \leq \frac1T \sum_t \left( f(x_t) - f(x^*) \right) \leq
                            \frac{DG}{\sqrt{T}} \leq \eps,
\end{align*}
as long as $T \geq \big(\frac{DG}{\eps}\big)^2$ and
$\eta = \frac{\epsilon}{G^2}$. 
 
  If the time horizon $T$ is unknown, setting a
time-dependent step size of $\eta_t = \frac{D}{G\sqrt{t}}$ works, with
an identical proof.  It is also well-known that the convergence bound
above is the best possible in general, modulo constant factors (see,
e.g., \cite[Thm~3.2.1]{Nest-book} or~\cite[Thm~3.13]{Seb}).

\subsubsection{Projected Gradient Descent}
\label{sec:proj-GD}

If we want to solve the constrained minimization problem for a convex
body $K$, we update as follows:
\begin{empheq}[box=\widefbox]{gather}
  x_{t+1}' \gets x_{t} - \eta \, \gr f_t(x_t) \\ \label{eq:gd-proj}
  x_{t+1} \gets \Pi_K(x'_{t+1}).
\end{empheq}
where $\Pi_K(x') := \arg\min_{x \in K} \norm{x - x'}$ is the projection of
$x'$ onto the convex body $K$. See Figure \ref{fig:pgd}.

\begin{figure}
\hfill\includegraphics[scale=0.7]{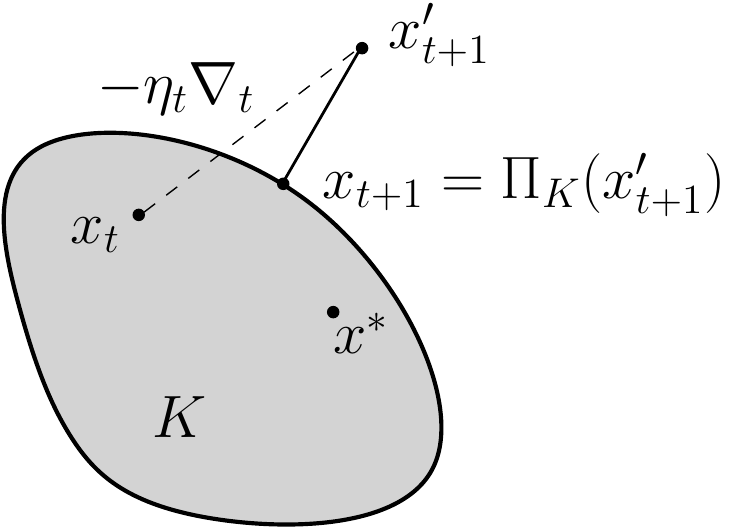}  \hfill\ %
\caption{Projected Gradient Descent}
\label{fig:pgd}
\end{figure}

\begin{proposition}[Pythagorean Property]
  \label{fct:pyth}
  Given a convex body $K \sse \R^n$, let $a \in K$ and $b' \in \R^n$. Let
  $b = \Pi_K(b')$. Then $\ip{a - b,b' - b} \leq 0$. Hence $\norm{a-b}^2
  \leq \norm{a-b'}^2$.
\end{proposition}
\begin{proof}
  For the first part, the separating hyperplane at $b$ has $b'$ on one side, and
  all of $K$ (and hence $a$) on the other side. Hence the angle between
  $b'-b$ and $a-b$ must be obtuse, giving the negative inner product.
  For the second part, $\norm{a-b'}^2 = \norm{a-b}^2 + \norm{b-b'}^2 +
  2\ip{a-b,b-b'}$. But the latter two terms are positive, which proves
  the lemma.
\end{proof}

Using this, we get that for any point $x^* \in K$,
\[ \norm{ x_{t+1} - x^* }^2 \leq \norm{ x_{t+1}' - x^* }^2 \]
Using the same potential function~(\ref{eq:pot}), this inequality implies:
\[ f_t(x_t) - f_t(x^*) + \Phi_{t+1} - \Phi_t \leq 
  f_t(x_t) - f_t(x^*) + \frac{1}{2\eta}(\norm{x_{t+1}' - x^*}^2 -
  \norm{x_t - x^*}^2),
\]
or in other words, the projection only helps and  we can follow the analysis from \S\ref{sec:basic-GD} starting
at~(\ref{eq:potchange-basic}) to bound the amortized cost by
$\frac{\eta G^2}{2}$. 
So this gives a regret bound identical to that of Theorem~\ref{thm:gd-basic-regret}.

\subsection{Strong Convexity Analysis}
\label{sec:strong-conv}

Let us a prove a better regret (and convergence) bound when the
functions are ``not too flat''. A function $f$ is
$\alpha$-\emph{strongly convex}, where $\alpha \geq 0$, if for all $u,v$
\begin{gather}
  f(\lambda u + (1 - \lambda)v) \leq \lambda f(u) + (1-\lambda) f(v) -
  \frac{\alpha}2 \lambda(1-\lambda) \|
  v-u\|^2 \label{eq:strong-nondiff}
\end{gather}
for all $\lambda \in [0,1]$. For the case of differentiable functions,
this implies that for all $x,y$, we have
\begin{gather}
  f(y) \geq f(x) + \ip{ \gr f(x), y-x } + \frac\alpha2 \|
  y-x\|^2. \label{eq:strong}
\end{gather}
For $\alpha$-strongly convex functions $f_t$, we use the same update
step but vary the step size $\eta_t$. Specifically,
\begin{gather}
  \boxed{x_{t+1} \gets x_{t} - \eta_t \, \gr f_t(x_t)},  \label{eq:gd-sc}
\end{gather}
where $\eta_t = \frac{1}{\alpha (t+1)}$. We only present a proof for the
unconstrained case, the constrained case follows as in
\S\ref{sec:proj-GD}.
\begin{theorem}[GD: Strong Convexity]
  \label{thm:gd-sc-regret}
  If the functions $f_t$ are $\alpha$-strongly convex and $G$ is an
  upper bound on $\norm{\gr f_t(x)}$ for all $x$, the update
  rule~(\ref{eq:gd-sc}) with $\eta_t = \frac{1}{\alpha (t+1)}$ guarantees
  an average regret of
  \[ \frac1T \sum_{t = 0}^{T-1} \big( f_t(x_t) - f_t(x^*) \big) \leq
    \frac{G^2 \log T}{2T\alpha}. \]
\end{theorem}

\begin{proof}
  The potential function is now
  \begin{gather}
    \boxed{\Phi_t = \frac{1}{2\eta_{t-1}}\norm{x_t -
        x^*}^2  = \frac{ t\alpha}{2} \norm{x_t -
        x^*}^2}  \label{eq:pot-sc}
  \end{gather}
  \textbf{The Change in Potential:}
  \begin{align}
    \Phi_{t+1} - \Phi_t &= \frac{\alpha (t+1)}{2}\norm{x_{t+1} - x^*}^2 -
                          \frac{\alpha t}{2}\norm{x_t - x^*}^2 \notag\\
                        &= \frac{\alpha}{2} \norm{x_{t} - x^*}^2 + \frac{1}{2\eta_{t}}
                          \left( \norm{x_{t+1} - x^*}^2  - \norm{x_{t} - x^*}^2 \right) \notag
    \\
                        &= \frac{\alpha}{2} \norm{x_{t} - x^*}^2 + \braket{\gr
                          f_t(x_t)}{x^* - x_t} + \frac{\eta_t}{2}\norm{\gr f_t(x_t)}^2 
                          \tag{by~(\ref{eq:potchange-basic})} 
  \end{align}
  \textbf{The Amortized Cost:}
  \begin{align}
    f_t(x_t) &- f_t(x^*) +  \Phi_{t+1} - \Phi_t \notag  \\
             &= \underbrace{f_t(x_t) - f_t(x^*) + \frac{\alpha}{2} \norm{x_{t} - x^*}^2 +
               \braket{\gr f_t(x_t)}{x^* - x_t}}_{\text{$\leq 0$ by $\alpha$-strong convexity}}  + \frac{\eta_t}{2}\norm{\gr
               f_t(x_t)}^2  \notag \\
             &\leq
               \frac{\eta_t}{2} \norm{
               \gr f_t(x_t) }^2 \quad \leq \quad \frac{\eta_t G^2}{2} \quad \textrm{(by bound on gradients)}
							\label{eq:potchange-sc}
  \end{align}
  Now summing over all time steps $t$, the total regret is
  \begin{align*}
    \sum_t \left( f_t(x_t) - f_t(x^*) \right)  
    &\leq \Phi_0 + \sum_t \frac{\eta_t}{2} G^2 \leq  0 + \frac{G^2 \log T}{2\alpha}.
  \end{align*}
Hence total regret only increases logarithmically as $\log T$ with time if the $f_t$ are strongly convex, as opposed to $\sqrt{T}$ in Theorem~\ref{thm:gd-basic-regret}. 
\end{proof}

This bound of $O(\log T)$ on the average regret is tight: Takimoto and
Warmuth~\cite{TakimotoW00} show a matching lower bound. However, in the
\emph{offline} optimization setting where we have a fixed function
$f_t=f$, using the same analysis but a better averaging shows a
convergence rate of $O(1/T)$ with respect to a convex combination of the
points $x_t$.

\begin{theorem}  
\label{thm:gd-sc-off}
Let $f$  be $\alpha$-strongly convex with gradients satisfying $\norm{\gr f(x)} \leq G $ for all $x$, and $x_t$ be the iterates produced by applying the update
  rule~(\ref{eq:gd-basic}) with $\eta_t = \frac{1}{\alpha t}$. For any $T\geq 1$, let $\overline{x}_T:= 
	\sum_{t=1}^T \lambda_t x_t$ denote the convex combination of $x_t$ with  $\lambda_t = \frac{2t}{T(T+1)}$. Then,
 \[f(\overline{x}_T) - f(x^*) \leq \frac{G^2}{\alpha (T+1)}.\] 
\end{theorem}
\begin{proof}
Instead of summing up \eqref{eq:potchange-sc} directly over $t$ in the regret analysis above, 
we first multiply \eqref{eq:potchange-sc} by $t$, and then sum over $t$ to obtain 
\[\sum_{t=1}^T  t (f_t(x_t) - f_t(x^*)) \leq \frac{1}{2\alpha } T G^2.\]
Using $f_t=f$ and dividing by $T(T+1)/2$ throughout, and by the convexity of $f$ we obtain 
\[f(\overline{x}_T) - f(x^*) \leq \frac{G^2}{\alpha (T+1)}. \qedhere\] 
\end{proof}

Finally, we remark that in the constrained case the same analysis with the same potential function works, exactly for same reason as in \ref{sec:proj-GD}.

\newpage

\section{Bounds for Smooth Functions}
\label{sec:smooth}

We now turn to the setting where the functions are Lipschitz smooth, i.e., when the gradient does not change too
rapidly. We know that in the online case, the average regret of
$O(1/\sqrt{T})$ is tight even for linear functions \cite{CBL}. However we get
better guarantees for the \emph{offline} setting where the function
$f_t = f$ for all time steps. The potential functions now look more
like~(\ref{eq:meta}), and use the difference $(f(x_t) - f(x^*))$ in
function value, not just in action space.

Define a function $f$ to be $\beta$-\emph{Lipschitz
  smooth} (or simply $\beta$-smooth) 
if for all $u,v$
\begin{gather}
  f(\lambda u + (1 - \lambda)v) \geq \lambda f(u) + (1-\lambda) f(v) -
  \frac{\beta}2 \lambda(1-\lambda) \|
  v-u\|^2 \label{eq:smooth-nondiff}
\end{gather}
for all $\lambda \in [0,1]$. For the case of differentiable functions,
this is equivalent to saying that for $x,y$, we have
\begin{gather}
  f(y) \,\balert{\leq}\, f(x) + \ip{ \gr f(x), y-x } + \frac\beta2 \|
  y-x\|^2. \label{eq:smooth}
\end{gather}
Observe the inequalities here are in the opposite directions from the
definitions of convexity~(\ref{eq:conv-def}) and
strong-convexity~(\ref{eq:strong}). Indeed, smoothness implies that the
function does not ``grow too fast'' anywhere. The smoothness condition
is equivalent to requiring that the gradients are Lipschitz continuous,
i.e., 
$\|\gr f(x) - \gr f(y)\|_2 \leq \beta \|x-y\|_2$ for all $x,y$.

\subsection{Smooth Gradient Descent}
\label{sec:eucl-smooth}


The update rule in this case has a time-invariant multiplier 
(where we use $\gr_t := \gr f(x_t)$ for brevity).
\begin{gather}
  \boxed{x_{t+1} \gets x_{t} - \frac{1}{\beta} \, \gr_t}. \label{eq:gd-smooth-update}
\end{gather}

We first show an analysis based on a very natural potential, that gives a slightly
sub-optimal bound with an additional $\log T$ factor. We improve this
later by slightly modifying the potential.

\begin{theorem}[Smooth Functions]
  \label{thm:smooth}
  If $f$ is $\beta$-smooth and
  $D := \max_x \{ \norm{x - x^*}_2 \mid f(x) \leq f(x_0) \}$, the update
  rule~(\ref{eq:gd-smooth-update}) guarantees
  \[ f(x_T) - f(x^*) \leq
    \beta\, \frac{D^2 (1 + \ln T)}{2T}. \]   
\end{theorem}

\begin{proof}
  To show a convergence rate of $O(1/t)$, perhaps the most natural
  approach is to consider the potential
  \[ \boxed{\Phi_t = t\cdot(f(x_t) - f(x^*))} \] and try to show that
  $\Phi_T = O(1)$.  This works, but gives a weaker bound of
  $\Phi_T = O(\log T)$ (note that conveniently, $\Phi_0 = 0$) and hence
  $f(x_T) - f(x^*) = O(\log T)/T$. Later we get rid of the
  logarithmic term.

  \textbf{The Potential Change:}
  \begin{align}
    \Phi_{t+1} - \Phi_t &= (t+1)(f(x_{t+1}) - f(x^*)) - t(f(x_t) - f(x^*))
                          \notag \\
                        &= (t+1)(f(x_{t+1}) - f(x_t)) + (f(x_t) - f(x^*)) \label{eq:pot-smooth}
  \end{align}
  To bound the first term, we use the smoothness of $f$ with $x = x_t$ and
  $y = x_{t+1} = x_t - \eta_t \gr_t$:
	  \begin{align}
    \textstyle f(x_{t+1}) \leq f(x_t) - \eta_t \cdot \| \gr_t \|_2^2 +
            \frac{\beta}{2} \cdot  \eta_t^2 \cdot \norm{\gr_t}_2^2. \notag
  \end{align}
  The choice of $\eta_t = 1/\beta$ minimizes the right hand side above to give 
  \begin{align}
    \textstyle f(x_{t+1}) \leq f(x_t) - \frac1{2\beta} \norm{\gr_t}_2^2. 
    \label{eq:sm-step} 
  \end{align}
		
  For the second term in~(\ref{eq:pot-smooth}), just use convexity and
  Cauchy-Schwarz:
  \begin{align}
    f(x_t) - f(x^*) \leq \ip{ \gr_t , x_t - x^* } &\leq \| \gr_t
    \|_2 \cdot \| x_t - x^*\|_2  \label{eq:2} \\ &\leq \nicefrac12\;( a \norm{\gr_t}^2 + (1/a) \norm{x_t -
      x^*}^2) , \notag
  \end{align}
  for any parameter $a > 0$.  Note that \eqref{eq:sm-step} ensures that
  $f(x_t) \leq f(x_{t-1}) \leq \cdots \leq f(x_0)$, so let us define $D
  := \max\{ \|x - x^*\|_2 \mid f(x) \leq f(x_0)\}$. 
  So the potential change is
  \begin{gather}
   \ts \Phi_{t+1} - \Phi_t \leq (t+1)\cdot(- \frac1{2\beta}) \| \gr_t
    \|_2^2 + \frac12 \; ( a \| \gr_t \|_2^2 + D^2/a).
  \end{gather}
  Choosing $a = \frac{t+1}\beta$ cancels the gradient terms. Hence the potential increase is
  at most $\frac{D^2 \beta}{2(t+1)}$, and 
  \[ f(x_T) - f(x^*) = \frac{\Phi_T}T = \frac1T \sum_{t = 0}^{T-1}
    (\Phi_{t+1} - \Phi_t) \leq \frac1T \sum_{t = 0}^{T-1} \frac{D^2}{2(t+1)} \beta \leq \beta\, \frac{D^2 (1+ \ln T)}{2T}. \qedhere\]
\end{proof}

The intuition is evident from~(\ref{eq:sm-step}) and~(\ref{eq:2}): we
improve a lot by~(\ref{eq:sm-step}) when the gradients are large, or
else we are close to the optimum by~(\ref{eq:2}). 

\paragraph{A Tighter Analysis.}

The logarithmic dependence in Theorem~\ref{thm:smooth} can be removed by
a simple trick of multiplying the potential by a linear term in $t$,
which avoids the sum over $1/t$.
\begin{theorem}[Smooth Functions: Take II]
  \label{thm:smoothII}
  If $f$ is $\beta$-smooth, and
  $D := \max_x \{ \norm{x - x^*}_2 \mid f(x) \leq f(x_0) \}$, the update
  rule~(\ref{eq:gd-smooth-update}) guarantees
  \[ f(x_T) - f(x^*) \leq
    \beta\, \frac{2D^2}{T+1}. \] 
\end{theorem}

\begin{proof}
  The potential now changes to
  \[ \boxed{\Phi_t = t(t+1)\cdot(f(x_t) - f(x^*))} \] The potential
  change is
  \begin{align*}
    \Phi_{t+1} - \Phi_t &= (t+1)(t+2) \cdot (f(x_{t+1}) - f(x_t)) + 2(t+1) \cdot
                          (f(x_t) - f(x^*))
                          \intertext{Plugging in~(\ref{eq:sm-step})
                          and~(\ref{eq:2}) gives}
    &\leq \ts (t+1)(t+2) \cdot ( - \frac{1}{2\beta} \norm{ \gr_t }^2 ) + 2(t+1)
      \cdot \| \gr_t \|_2\cdot D \leq 2D^2 \beta\cdot \frac{t+1}{t+2},
  \end{align*}
  where the last inequality is the maximum value of the preceding
  expression obtained at $\norm{\gr_t} = \frac{2 \beta D}{t+2}$. Summing
  over the time steps, $\Phi_T \leq T \cdot 2D^2 \beta$, so
  \[f(x_T) - f(x^*) \leq \frac{2D^2 \beta \cdot T}{T(T+1)} = \beta\,
  \frac{2D^2 }{T+1}.\qedhere\]
\end{proof}

\subsubsection{Yet Another Proof}
\label{sec:takeIII}

Let's see yet another proof that gets rid of the logarithmic
term. Interestingly, the potential function now combines both the
difference in the function value, and the distance in the ``action''
space.

\begin{theorem}[Smooth Functions: Take III]
  \label{thm:smoothIII}
  If $f$ is $\beta$-smooth, the update rule~(\ref{eq:gd-smooth-update})
  guarantees
  \[ f(x_T) - f(x^*) \leq
    \beta\, \frac{\norm{x_0 - x^*}^2}{2T}. \] 
\end{theorem}

\begin{proof} 
  Consider the potential of the form
  \[ \boxed{ \Phi_t = t\, (f(x_t) - f(x^*)) + \frac{a}{2} \|x_t - x^*\|^2 } \] where
  $a$ will be chosen based on the analysis below. As
  $\Phi_0 = \frac{a}{2} \|x_0-x^*\|^2$, if we show that $\Phi_t$ is non-increasing,
  \[   \frac{a}{2} \|x_0-x^*\|^2  = \Phi_0 \geq \Phi_T = T  (f(x_t) - f(x^*)) + \frac{a}{2} \|x_t - x^*\|^2  \]
  which gives $f(x_t) - f(x^*) \leq \frac{a}{2T}\|x_0-x^*\|^2$ as desired.

  The potential difference can be written as:
  \begin{align}
  \Phi_{t+1} - \Phi_t & =   (t+1) \underbrace{(f(x_{t+1}) -
      f(x_t))}_{(\ref{eq:sm-step})} + \underbrace{f(x_t) -
      f(x^*)}_{\text{(convexity)}} + \frac{a}{2}\, \underbrace{(\norm{x_{t+1} - x^*}^2 -
                        \norm{x_t - x^*}^2
                        )}_{(\ref{eq:potchange-basic})}. \label{eq:18} \\
    \intertext{Using the bounds from the mentioned inequalities, }
   \leq\ & (t+1) \cdot \overbrace{\ts -\frac{1}{2\beta} \norm{\gr_t}_2^2} +
    \overbrace{\ip{\gr_t,x_t - x^*}}  + \frac{a}{2} \big( 2\overbrace{\ts
           \eta_t \ip{\gr_t, x^* -
        x_t } + \eta_t^2 \norm{\gr_t}_2^2} \big) \label{eq:20}
  \end{align}
  where $\eta_t = \nicefrac{1}{\beta}$ in this case. Now, we set
  $a = 1/\eta_t = \beta$ to cancel the inner-product terms, which gives
  \[ \Phi_{t+1} - \Phi_t \leq -(\nicefrac{t}{2\beta}) \norm{\gr_t}^2
  \leq 0. \qedhere \]
\end{proof}
This guarantee is almost the same as in Theorem~\ref{thm:smoothII}, with
a slightly better definition of the distance term ($\|x_0-x^*\|^2$ vs $D^2$). But we will revisit
and build on this proof when we talk about Nesterov acceleration in
\S\ref{sec:nesterov}.

\subsubsection{Projected Smooth Gradient Descent}
\label{sec:proj-smooth}

We now consider the constrained minimization problem for a convex body
$K$. As previously, the update involves taking a step and then
projecting back onto $K$:
\begin{empheq}[box=\widefbox]{gather}
  x_{t+1}' \gets x_{t} - (1/\beta) \, \gr f(x_t) \notag\\ 
  x_{t+1} \gets \Pi_K(x'_{t+1}). \label{eq:gd-proj-smooth}
\end{empheq}

\begin{theorem}[Constrained Smooth Optimization]
  \label{thm:smooth-proj}
  If $f$ is $\beta$-smooth, the update rule~(\ref{eq:gd-proj-smooth})
  guarantees
  \[ f(x_T) - f(x^*) \leq \frac{ \frac\beta2 \,\norm{x_0 -
        x^*}^2}{T}. \]
\end{theorem}

\begin{proof} 
  We use the same potential as in Theorem~\ref{thm:smoothIII}:
  \[ \boxed{ \Phi_t =  t (f(x_t) - f(x^*)) + \frac{\beta}{2} \|x_t
      - x^*\|^2 }. \]
  The potential difference is now written as:
  \begin{align}
    \Phi_{t+1} - \Phi_t & =   t \underbrace{(f(x_{t+1}) -
                          f(x_t))}_{} + \underbrace{f(x_{t+1}) -
                          f(x^*)}_{} + \ts \frac{\beta}{2}\, \underbrace{(\norm{x_{t+1} - x^*}^2 -
                          \norm{x_t - x^*}^2
                          )}_{\norm{a}^2 - \norm{a+b}^2 = -2\ip{a,b} -
                          \norm{b}^2 }. \label{eq:19}
  \end{align}
  As $x_{t+1}$ is the projected point, we cannot directly
  use~\eqref{eq:sm-step} to bound the first and second terms, but we can
  show the following claim (which we prove later) that follows from
  smoothness:
  \begin{claim}
    \label{clm:proj-magic}
    For any $y \in K$, 
    $f(x_{t+1}) - f(y) \leq  \beta \ip{x_t - x_{t+1},x_t - y} - \nicefrac{\beta}{2} \norm{x_t - x_{t+1}}^2$
  \end{claim}
  Using Claim~\ref{clm:proj-magic} to bound the first and second terms
  of~(\ref{eq:19}), we get
  \begin{align*}
    &\leq  t \cdot \overbrace{0}
                           + \overbrace{\beta \ip {x_t - x_{t+1},x_t -
                           x^*} -
                           \nicefrac{\beta}{2} \norm{x_t - x_{t+1}}^2} - \ts \nicefrac{\beta}{2} \big( \overbrace{\ts
                           2\ip{x_t-x_{t+1}, x_{t+1} - x^*} + \norm{x_t - x_{t+1}}^2} \big) 	\\ 
                         & =  \beta \ip {x_t - x_{t+1},x_t - x_{t+1}} -
                           \beta \norm{x_t - x_{t+1}}^2  =0.
  \end{align*}
  This completes the proof.
\end{proof}

\begin{proof}[Proof of Claim~\ref{clm:proj-magic}]
  We write $ f(x_{t+1}) - f(y)  = (f(x_{t+1}) - f({x_t})) +
  (f(x_t) - f(y))$. Now using smoothness and convexity for the
first and second terms respectively, we have
\begin{align}f(x_{t+1}) - f(y) &\leq \overbrace{\ip{\gr_t,x_{t+1} - x_t}
                                 + \nicefrac{\beta}{2}
                                 \norm{x_{t+1}-x_t}^2} +\overbrace{
                                 \ip{\gr_t,x_t-y}} \notag \\
  &=  \ip{\gr_t,x_{t+1} - y} +  \nicefrac{\beta}{2} \norm{x_{t+1}-x_t}^2  \label{eq:proj-mag-smooth}
\end{align}
Since $\beta \ip{x_{t+1} - x'_{t+1},x_{t+1}-y} \leq 0$ by the
Pythagorean property Prop.~\ref{fct:pyth},
\begin{align*}
\ip{\gr_t,x_{t+1} - y} &= \ip{\beta (x_t -x'_{t+1}), x_{t+1} -y} \leq
  \beta \ip{x_t - x_{t+1},x_{t+1}-y} \\
&  = \beta \ip{x_t - x_{t+1},x_{t}-y} - \beta \norm{x_{t+1}-x_t}^2
\end{align*}
Substituting into  \eqref{eq:proj-mag-smooth} gives the result.
\end{proof}

\subsubsection{The Frank-Wolfe Method}
\label{sec:frank-wolfe}

One drawback of projected gradient descent is the projection step: given
a point $x'$ and a body $K$, finding the closest point $\Pi_K(x')$ might
be computationally expensive. Instead, we can use a different rule, the
\emph{Frank-Wolfe} method (also called \emph{conditional gradient
  descent}) \cite{FW56}, that implements each gradient step using linear optimization
over the body $K$. Loosely, at each timestep we find the point in $K$
that is furthest from the current point in the direction of the negative
gradient, and move a small distance towards it.

\begin{figure}[H]
\hfill\includegraphics[scale=0.5]{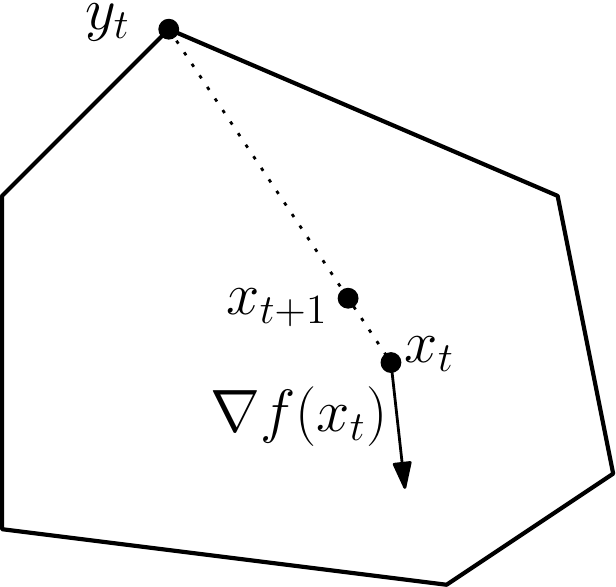}  \hfill\ %
\caption{The Frank-Wolfe Update}
\label{fig:fw}
\end{figure}

Formally, the update rule for Frank-Wolfe method is simple:
\begin{empheq}[box=\widefbox]{gather}
  y_t \gets \arg\min_{y \in K} \;\ip{\gr_t, y} \notag \\
  x_{t+1} \gets (1 - \eta_{t}) x_{t} + \eta_t y_t \label{eq:gd-FW}
\end{empheq}
Setting $\eta_t = 1/(t+1)$ in hindsight will give the following result.

\begin{theorem}[Smooth Functions: Frank-Wolfe]
  \label{thm:frankw}
  If $f$ is $\beta$-smooth, $K$ is a convex body with $D := \max_{x,y\in
    K} \norm{x-y}$, then the update rule~(\ref{eq:gd-FW})
  guarantees
  \[ f(x_T) - f(x^*) \leq
    \beta\, \frac{D^2 \, (1 + \ln T)}{2T}. \] 
\end{theorem}

\begin{proof}
  We use the simplest potential function from Theorem~\ref{thm:smooth}:
  \[ \boxed{\Phi_t = t\cdot(f(x_t) - f(x^*))}, \] and hence the change
  in potential is again:
  \begin{gather}
    \Phi_{t+1} - \Phi_t = (t+1)(f(x_{t+1}) - f(x_t)) + (f(x_t) -
    f(x^*)) \label{eq:17}
  \end{gather}
  To bound the change in potential~(\ref{eq:17}), we observe
  that $x_{t+1} - x_t = \eta_{t} (y_t - x_t)$.
  \begin{align*}
    f(x_{t+1}) - f(x_t) &\leq \ip{\gr_t, x_{t+1} - x_t} +
                          \frac{\beta}{2} \norm{x_{t+1} - x_t}^2 \tag{by smoothness}\\
                        &= \eta_t \ip{\gr_t, y_t
                          - x_t} + \frac{\beta\eta_t^2}{2} \norm{y_t- x_t}^2.\\
                        &\leq \eta_t \ip{\gr_t, x^*
                          - x_t} + \frac{\beta\eta_t^2}2 \norm{y_t- x_t}^2.
                          \tag{by optimality of $y_t$}\\
    f(x_t) - f(x^*) &\leq \ip{\gr_t, x_t - x^*} \tag{by convexity}
  \end{align*}
  Setting $\eta_t := \frac{1}{t+1}$ cancels the linear terms and hence
  the potential change~(\ref{eq:17})
  is at most $\beta\eta_t D^2/2$. Summing over $t$ and using
  $\Phi_0 = 0$, the final potential $\Phi_T \leq D^2 (1+\ln T)$, and hence
  $f(x_T) - f(x^*) = \beta\;\frac{(1 + \ln T)\cdot D^2}{2T}$.
\end{proof}

We can remove the logarithmic dependence in the error by multiplying the
potential by $(t+1)$ as in Theorem~\ref{thm:smoothII}; this gives the
following theorem, whose simple proof we omit.

\begin{theorem}[Smooth Functions: Frank-Wolfe, Take II]
  \label{thm:frankw2}
  If $f$ is $\beta$-smooth, $K$ is a convex body with $D := \max_{x,y\in
    K} \norm{x-y}$, then the update rule~(\ref{eq:gd-FW}) with $\eta_t = 2/(t+1)$
  guarantees
  \[ f(x_T) - f(x^*) \leq 2 \beta\, \frac{D^2 }{T+1}. \]
\end{theorem}

\subsection{Well-Conditioned Functions}
\label{sec:well-cond}

If a function is both $\alpha$-strongly convex and $\beta$-smooth, it
must be that $\alpha \leq \beta$. The ratio $\kappa := \beta/\alpha$ is
called the \emph{condition number} of the convex function. We now show a
much stronger convergence guarantee for ``well-conditioned'' functions,
i.e., functions with small $\kappa$ values. The update rule is the same
as for smooth functions:
\begin{gather}
  \boxed{x_{t+1} \gets x_{t} - \frac{1}{\beta} \, \gr_t}.  \label{eq:gd-wellcond}
\end{gather}

\begin{theorem}[GD: Well-Conditioned]
  \label{thm:smooth-wellcond}
  Given a function $f$ that is both $\alpha$-strongly convex and
  $\beta$-smooth, define $\kappa := \beta/\alpha$.	
	The update
  rule~(\ref{eq:gd-wellcond}) ensures
  \begin{gather}
    f(x_T) - f(x^*) \leq \exp(-T/\kappa) \cdot ( f(x_0) - f(x^*)
    ) \quad\quad\text{for all $x^*$.} \notag
  \end{gather}
\end{theorem}

\begin{proof}
  We set $\gamma = 1/(\kappa-1)$ for brevity\footnote{Note that
    $\kappa=1$ iff $f(x)= a x^\intercal x + b^\intercal x +c$ for
    suitable scalars $a,c$ and $b \in \R^n$; in this case it is easily checked that the optimum solution $x^*$  is reached in a single step.}, and use the potential
  \begin{gather}
    \boxed{ \Phi_t = (1+\gamma)^t \cdot (f(x_t) -
      f(x^*))} \label{eq:wellcond}
  \end{gather}
  This is a natural potential to use, as we wish to show that $f(x_T) -f(x_0)$ falls exponentially with $T$. 

  \textbf{The Potential Change:} A little rearrangement gives us
  \begin{gather}
    \Phi_{t+1} - \Phi_t = (1+\gamma)^t \cdot \bigg( (1+\gamma)
    \big(f(x_{t+1}) - f(x_t)\big) + \gamma \big(f(x_t) - f(x^*)\big)
    \bigg). \label{eq:wellcond-potchange}
  \end{gather}
  We bound the two terms separately.  Using the smoothness analysis
  from~(\ref{eq:sm-step}):
  \[ f(x_{t+1}) - f(x_t) \leq - \frac1{2\beta} \norm{ \gr_t}^2. \] And by the definition of strong convexity,
  \[ f(x_t) - f(x^*) \leq \ip {\gr_t, x_t - x^*} -
    \frac{\alpha}{2}\norm{x_t-x^*}^2 \leq \frac{1}{2\alpha}
    \norm{\gr_t}^2
  \]
  where the second inequality uses
  $\ip{a, b} - \|b\|^2/2 \leq \|a\|^2/2 $.  Plugging this back
  into~\eqref{eq:wellcond-potchange} gives
  \[ (1+\gamma)^t \left( - \frac{1+\gamma}{2\beta} + \frac{\gamma}{2
        \alpha} \right ) \norm{\gr_t}^2\] which is $0$
  by our choice of $\gamma$. Hence, after $T$ steps,
  \begin{align}
    f(x_T) - f(x^*) &\leq (1 + \gamma)^{-T} (f(x_0) - f(x^*)) =
    (1-\nicefrac{1}{\kappa})^T (f(x_0)-f(x^*)) \notag \\
    &\leq
    e^{-T/\kappa} (f(x_0) - f(x^*)). \label{eq:1} 
  \end{align}
  Hence the proof.
	\end{proof}
	
  Here we can show that the algorithm's point $x_T$ also gets rapidly closer to $x^*$. If $x^*$ is the optimal point, we know $\gr f(x^*) = 0$. Now
  smoothness gives 
  $f(x_0) - f(x^*) \leq \frac{\beta}{2} \norm{x_0 - x^*}^2$, and strong
  convexity gives
  $\frac{\alpha}{2} \norm{x_T - x^*}^2 \leq f(x_T) - f(x^*)$. Plugging
  into~(\ref{eq:1}) gives us that
  \[ \norm{x_T - x^*}^2 \leq \kappa e^{-T/\kappa}\cdot \norm{x_0 -
      x^*}^2. \qedhere \]

  A few remarks. Firstly, Theorem~\ref{thm:smooth-wellcond} implies that
  reducing the error by a factor of $1/2$ can be achieved by increasing
  $T$ \emph{additively} by $\kappa \ln 2$. Hence if the condition number
  $\kappa$ is constant, every constant number of rounds of gradient
  descent gives us one additional bit of accuracy! This behavior, where
  getting error bounded by $\eps$ requires $O(\log \eps^{-1})$ steps, is
  called \emph{linear convergence} in the numerical analysis literature.

  One may ask if the convergence for smooth, and for well-conditioned
  functions is optimal as a function of $T$. The answer is no: a famed
  result of Nesterov gives faster (and optimal) convergence rates. We
  see this result and a potential-function-based proof in \S\ref{sec:nesterov}.

  Finally, the proof of Theorem~\ref{thm:smooth-wellcond} can be
  extended to the constrained case using the same potential function and
  the update rule~\eqref{eq:gd-proj-smooth}, but now using an analog of
  Claim~\ref{clm:proj-magic} that shows that for any $y \in K$,
  \[ f(x_{t+1}) - f(y) \leq \beta \ip{x_t - x_{t+1},x_t - y} -
    \nicefrac{\beta}{2} \norm{x_t - x_{t+1}}^2 - \nicefrac{\alpha}{2}
    \norm{x_t - y}^2. \] We omit the simple proof.


\newpage
\section{The Mirror Descent Framework}
\label{sec:MD} 

The gradient descent algorithms in the previous sections work by adding
some multiple of the gradient to current point. However, this should
strike the reader as somewhat strange, since the point $x_t$ and the
gradient $\gr f(x_t)$ are objects that lie in different spaces and
should be handled accordingly. In particular, if $x_t$ lies in some
vector space $E$, the gradient $\gr f(x_t)$ lies in the dual vector
space $E^*$. (This did not matter earlier since $\R^n$ equipped with the
Euclidean norm is self-dual, but now we want to consider general norms
and would like to be careful.)

A key insight of Nemirovski and Yudin~\cite{NY} was that substantially more
general and powerful results can be obtained, without much additional
work, by considering these spaces separately. 
For example, it is well-known (and we will show) that the classic multiplicative-weights update
method can be obtained as a special case of this general approach.

\subsection{Basic Mirror Descent}
\label{sec:basic-MD}

The key idea in mirror descent is to define an injective mapping between $E$ to
$E^*$, which is called the \emph{mirror map}. Given a point $x_t$, we first map
it to $E^*$, make the gradient update there, and then use the inverse
map back to obtain the point $x_{t+1}$.

We start some basic concepts and notation. Consider some vector space
$E$ with an inner product $\langle \cdot, \cdot \rangle$, and define a
norm $\norm{\cdot}$ on $E$. To measure distances in $E^*$, we use the
\emph{dual norm}  defined as
\begin{gather}
  \norm{y}_* := \max_{x : \norm{x} = 1} \ip{x,y}. \label{eq:dual-norm}
\end{gather}
By definition we
have
\begin{gather}
  \ip{x,y} \leq \norm{x}\cdot \norm{y}_*, \label{eq:gen-cs}
\end{gather}
which is often referred to as the \emph{generalized Cauchy-Schwarz}
inequality.

A function $h$ is \emph{$\alpha$-strongly convex} with respect to
$\norm{\cdot}$ if
\[ h(y) \geq h(x) + \ip{\gr h(x),y-x} + \frac{\alpha}{2} \|y-x\|^2 \]
Such a strongly convex function $h$ defines a map from $E$ to $E^*$ via
its gradient: indeed, the map $x \mapsto \gr h(x)$ takes the point
$x \in E$ into a point in the dual space $E^*$.  The strong convexity
ensures that the map is $1$-$1$ (i.e., $\gr h(x) \neq \gr h(y) $ for
$x \neq y$).  Moreover, the map $\gr h(\cdot)$ is also surjective, so
for any $\theta \in E^*$ there is an inverse $x \in E$ such that
$\gr h(x) = \theta$. In fact, this inverse map is given by the gradient
of the Fenchel dual for $h$, i.e.,
$\gr h^*(\theta) = x \iff \gr h(x) = \theta$. (For the reader not
familiar with Fenchel duality, it suffices to interpret
$\gr h^*(\theta)$ merely as $(\gr h)^{-1}(\theta)$.) Readers interested
in the technical details can see, e.g.,~\cite{BT03} or \cite[Chapter~7]{bertsekas03}.

\subsubsection{The Update Rules}

\begin{figure}
  \begin{center}
    \includegraphics[scale=0.8]{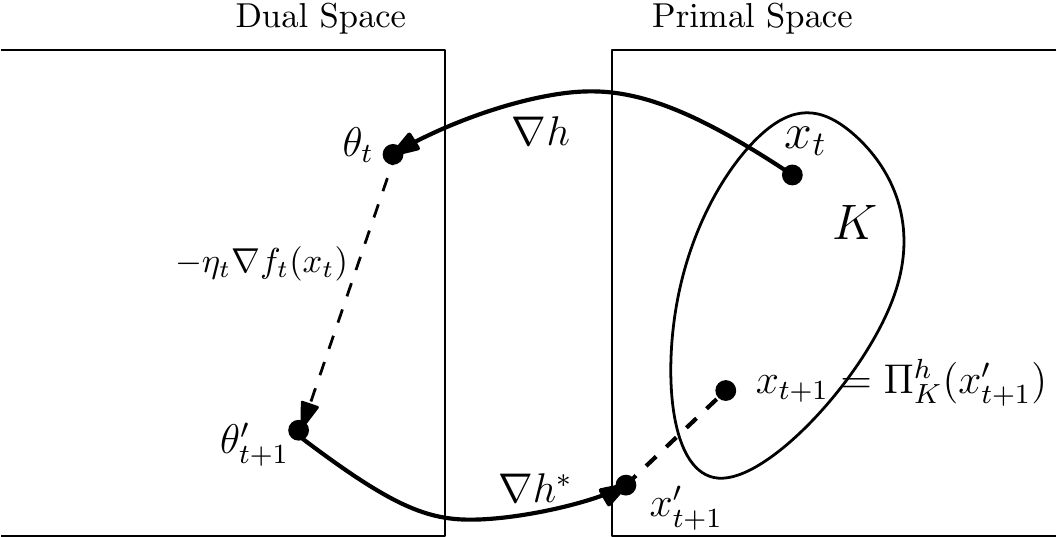}
  \end{center}
\caption{Mirror Descent}
\label{fig:md}
\end{figure}

Since $\gr h: E \to E^*$ gives us a map from the primal space $E$ to the
dual space $E^*$, we keep track of the image point
$\theta_t = \gr h(x_t)$ as well. Now, the updates are the natural ones,
given by
\begin{empheq}[box=\widefbox]{align}
  \theta_{t+1}' &= \theta_t - \eta_t \gr f_t(x_t) \label{eq:md} \\
	x'_{t+1} & = \gr h^*(\theta_{t+1}') \notag \\
  x_{t+1} &= \arg\min_{x \in K} \hdiv{x}{x_{t+1}'}. \notag
\end{empheq}
In other words, given $x_t \in E$, we add $\eta_t$ times the negative gradient to its
image $\theta_t = \gr h(x_t)$ in the dual space to get $\theta'_{t+1}$,
pull the result back to $x_{t+1}' \in E$ (using the inverse mapping
$x_{t+1}' = \gr h^*(\theta'_{t+1})$), and project it back onto $K$ to
get $x_t$.  Of course, we may not want to use the Euclidean distance for
the projection; the ``right'' distance in this case is the
\emph{Bregman divergence}, which we discuss
shortly. 

An equivalent way to present the mirror descent update is the following:
\begin{empheq}[box=\widefbox]{align}
  x_{t+1} &= \arg\min_{x \in K} \bigg\{ \ip{ \eta_t \gr f_t(x_t), x - x_t } +
  \hdiv{x}{x_t} \bigg\}. \label{eq:md-prox}
\end{empheq}
This is generalization of~(\ref{eq:gd-proxform}). The equivalence is
easy to see in the unconstrained case (just take derivatives), for
the constrained case one uses the KKT conditions.

\subsubsection{Bregman Divergences}

Given a strictly convex function $h: \R^d \to \R$, define the
\emph{Bregman divergence}
\[ \hdiv{y}{x} := h(y) - h(x) - \ip{\gr h(x), y-x} \] to be the
``error'' at $y$ between the actual function value and the value given
by linearization at some point $x$.  The convexity of $h$ means this
quantity is non-negative; if $h$ is $\beta$-strongly convex with respect
to the norm $\norm{\cdot}$, then
$\hdiv{y}{x} \geq \frac{\beta}{2} \norm{y-x}^2$. Also, $\hdiv{y}{x}$ is
a convex function of $y$ (for a fixed $x$, this is a convex function
minus a linear term), and the gradient of the divergence with respect to
the first argument is $\gr_y (\hdiv{y}{x}) = \gr h(y) - \gr
h(x)$.

For example, the function $h(x) := \frac12 \norm{x}_2^2$ is $1$-strongly
convex with respect to $\ell_2$ (and hence strictly convex), and the
associated Bregman divergence $\hdiv{y}{x} = \frac12 \norm{y - x}_2^2$,
half the squared $\ell_2$ distance. This distance is not a metric, since
it does not satisfy
the triangle inequality. Or consider the \emph{negative entropy
  function} $h(x):= \sum_i x_i \ln x_i$ defined on the probability
simplex $\triangle_n := \{ x\in [0,1]^n \mid \sum_i x_i=1\}$. For
$x, y \in \triangle_n$, the associated Bregman divergence $\hdiv{y}{x}$
is $\sum_i y_i \ln \nicefrac{y_i}{x_i}$, the relative entropy or \emph{Kullback-Leibler (KL)
  divergence}  from $x$ to $y$ .  This distance is not even symmetric in
$x$ and $y$.

\paragraph{Bregman projection.} Given a convex body $K$ and a strictly
convex function $h$, we define the \emph{Bregman projection} of a point
$x'$ on $K$ as
\[\Pi^h_K(x') = \argmin_{x \in K} \hdiv{x}{x'}. \]
If $x'\in K$, then $\Pi^h_K(x')=x'$ because $\hdiv{x'}{x'}=0$. For
$h(x) = \frac12 \norm{x}^2$, this corresponds to the usual Euclidean
projection.  A very useful feature of Bregman projections is that they
satisfy a ``Pythagorean inequality'' with respect to the divergence,
analogous to Fact~\ref{fct:pyth}.

\begin{proposition}[Generalized Pythagorean Property]
  \label{fct:gpyth}
  Given a convex body $K \sse \R^n$, let $a \in K$ and $b' \in \R^n$. Let
  $b = \Pi^h_K(b')$. Then
  \[ \ip{\gr h(b') - \gr h(b),a - b} \leq 0. \]  In
  particular, 
  \begin{gather}
    \hdiv{a}{b'} \geq \hdiv{a}{b} + \hdiv{b}{b'}, \label{eq:gen-pyth}
  \end{gather}
  and hence $\hdiv{a}{b} \leq \hdiv{a}{b'}$.
\end{proposition}
\begin{proof}
  Recall that for any convex function $g$ and convex body $K$, if
  $x^* = \argmin_{x\in K} g(x)$ is the minimizer of $g$ in $K$, then
  $\ip{\gr g(x^*), y-x^*} \geq 0$ for all $y\in K$.  Using
  $g(x) = \hdiv{x}{b'}$, and noting that $g(x)$ is convex with
  $\gr g(x) = \gr h(x) - \gr h(b')$ and that the minimizer $x^*= b$,
  we get $\ip{ \gr h(b) - \gr h(b'), a-b } \geq 0$ for all $a \in K$.
	
  For the second part, expand the terms using the definition of
  $\hdiv{a}{b}$ and cancel the common terms, the desired inequality
  turns out to be equivalent to
  $\ip{\gr h(b') - \gr h(b) , a-b } \leq 0$. The last inequality uses
  that the divergences are non-negative.
\end{proof}

\subsubsection{The Analysis}

We consider the more general \emph{online} optimization setting, and prove the following regret bound.

\begin{theorem}
  \label{thm:mirror}
  Let $K$ be a convex body, $f_1,\ldots,f_T$ be convex functions defined
  on $K$, $\|\cdot\|$ be a norm, and $h$ be an $\alpha_h$-strongly convex
  function with respect to $\|\cdot\|$. The mirror descent algorithm
  starting at $x_0$ and taking constant step size $\eta_t = \eta$ in
  every iteration, produces $x_1,\ldots,x_T$ such that
  \begin{equation}\label{eq:main}
    \sum_{t=1}^{T}f_t(x_t) - \sum_{t=1}^n f_t(x^*) \leq 
    \frac{\hdiv{x^*}{x_0}}{\eta}+ \frac{\eta \sum_{t=1}^T{\|\gr
        f_t(x_t)\|^2_*}}{2\alpha_h} \quad\text{ , for all $x^* \in K$} 
  \end{equation} 
\end{theorem}

\begin{proof}
  Define the potential
  \begin{gather}
    \boxed{ \Phi_t = \frac{\hdiv{x^*}{x_t}}{\eta}. }
  \end{gather}
  Observe that plugging in $h(x) = \frac12 \norm{x}_2^2$ gives us the potential
  function~(\ref{eq:pot}) for the Euclidean norm.

  \textbf{The Potential Change:} For brevity, use $\gr_t := \gr f_t(x_t)$.
  \begin{align}
     \hdiv{x^*}{x_{t+1}} &- \hdiv{x^*}{x_t} \notag \\
    &\leq \hdiv{x^*}{x_{t+1}'}- \hdiv{x^*}{x_t}
      \tag{generalized Pythagorean ppty.} \notag\\
    &= h(x^*)-h(x_{t+1}')-\langle\underbrace{\gr
      h(x_{t+1}')}_{\theta'_{t+1}},x^*-x_{t+1}'\rangle -h(x^*)+h(x_t)+
      \langle\underbrace{ \gr h(x_{t})}_{\theta_t},x^*-x_{t}\rangle
      \notag \\ 
    &= h(x_t)-h(x_{t+1}')-\langle \theta'_{t+1},x_t-x_{t+1}'\rangle - 
      \langle \theta'_{t+1} - \theta_t,x^*-x_{t}\rangle
      \notag \\
    &= \underbrace{h(x_t)-h(x_{t+1}') -\langle \theta_{t},x_t-x_{t+1}'\rangle }_{\textrm{strong convexity}}
      +\langle \eta_t \gr_t ,x_t-x_{t+1}'\rangle +
      \langle \eta_t \gr_t,x^*-x_{t}\rangle
      \notag  \\
    &\leq - \frac{\alpha_h}{2} \|x_{t+1}'-x_t\|^2 + \eta_t \langle \gr_t,
      x_t -x_{t+1}'\rangle + \eta_t \langle \gr_t,
      x^*-x_{t}\rangle \notag\\
      &\leq \frac{\eta_t^2}{2\alpha_h} \norm{ \gr_t }_*^2 +
        \eta_t \langle \gr_t,
      x^*-x_{t}\rangle. \label{eq:pot-change-md}
  \end{align}
  The last inequality uses generalized Cauchy-Schwarz to get
  $\ip{a, b} \leq \|b\| \|a\|_* \leq \|b\|^2/2 + \|a\|_*^2/2$.  Observe
  that~(\ref{eq:pot-change-md}) precisely maps
  to~(\ref{eq:potchange-basic}) when we consider the Euclidean norm.

  \textbf{The Amortized Cost:} Recall that we set $\eta_t = \eta$ for
  all steps. Hence, dividing~(\ref{eq:pot-change-md}) and substituting,
  \begin{align*}
    f_t(x_t)-f_t(x^*)+(\Phi_{t+1}-\Phi_t )
    &\leq \underbrace{{f_t(x_t)-f_t(x^*)}+ \langle
      \gr_t,x^*-x_t\rangle}_{\text{$\leq 0$ by convexity of $f_t$}}
    + \frac{\eta}{2\alpha_h} \norm{\gr_t}_*^2.
  \end{align*}
  The total regret then becomes
  \begin{gather*}
    \sum_{t} (f_t(x_t)- f_t(x^*)) 
    \leq \Phi_0 +\sum_t \frac{\eta}{2\alpha_h}\|\gr_t\|_*^2
    \leq 
    \frac{\hdiv{x^*}{x_0}}{\eta}+\frac{\eta\sum_{t=1}^{T}\|\gr_t\|_*^2}{2\alpha_h}. 
  \end{gather*}
  Hence the proof.
\end{proof}

\subsubsection{Special Cases}

To get some intuition, let us look at some well-known special cases.
If we use the $\ell_2$ norm, and $h(x) := \frac12 \norm{x}_2^2$ which is
clearly $1$-strongly convex with respect to $\ell_2$, the associated
Bregman divergence $\hdiv{x^*}{x} = \frac12 \norm{x^* -
  x}_2^2$. Moreover, the Euclidean norm is self-dual, so if we bound
$\norm{ \gr f_t}_2$ by $G$, the total regret bound above is
$\frac{1}{2\eta} \norm{ x^* - x_0 }_2^2 + \eta T G^2/2$. This is the
same result for projected gradient descent we derived in
Theorem~\ref{thm:gd-basic-regret}---and in fact the algorithm is also
precisely the same.

Now consider the $\ell_1$ norm, with $K$ being the probability simplex
$\triangle_n := \{ x \in [0,1]^n \mid \sum_i x_i = 1\}$. If we choose
the negative entropy function $h(x) := \sum_i x_i \ln x_i$, then
$\hdiv{x^*}{x}$ is just the well-known Kullback-Liebler
divergence. Moreover, Pinsker's inequality says that
$\KLdiv{p}{q} \geq \norm{p-q}_1^2$, which implies that $h$ is
$1$-strongly convex with respect to $\ell_1$. Applying Theorem \ref{thm:mirror} now gives a regret bound of
\[ \frac{\KLdiv{x^*}{x_0}}{\eta} + \frac{\eta}{2} \sum_t
  \norm{\gr_t}_\infty^2. \] Let's also see what the mirror-descent
algorithm does in this case. The mirror map takes the point $x$ to
$\gr h(x) = (1 + \log x_i)_i$, and the inverse map takes $\theta$ to
$\gr h^*(\theta) = (e^{\theta_i - 1})_i$. This point may be outside the
probability simplex, so we do a Bregman projection, which in this case
corresponds to just a rescaling $x \mapsto \nicefrac{x}{\norm{x}_1}$.
Unrolling the process, one can get a closed-form expression for the
point $x_T$:
\[ (x_T)_i = \frac{(x_0)_i \cdot \exp\{\sum_t (\gr_t)_i\}}{\sum_j
    (x_0)_j \cdot \exp\{\sum_t (\gr_t)_j\}}. \] 

E.g., if we specialize even further to online \emph{linear}
optimization, where each function $f_t(x) = \ip{\ell_t, x}$ for some
$\ell_t \in [0,1]^n$, the gradient is $\ell_t$ and its
$\ell_\infty$-norm is $\| \ell_t \|_\infty \leq 1$, giving us the
familiar regret bound of
$ \frac{\KLdiv{x^*}{x_0}}{\eta} + \frac{\eta T}{2}$ that we get from the
multiplicative weights/Hedge algorithms.  Which is not surprising, since
this algorithm is precisely the Hedge algorithm!

\subsection{An Aside: Smooth Functions and General Norms}
\label{sec:general-norms}

Let us consider minimizing a function that is smooth with respect to
non-Euclidean norms, in the unconstrained case.
When we consider an arbitrary norm $\norm{\cdot}$, the definition of a
smooth function~(\ref{eq:smooth}) extends seamlessly. Now we can define
an update rule by naturally extending~(\ref{eq:gd-proxform}):
\begin{gather}
  \boxed{x_{t+1} \gets \arg\min_{x} \Big\{ \frac{1}{2} \norm{ x -
      x_t }^2 + \eta_t \ip{ \gr_t, x - x_t }
    \Big\}}, \label{eq:gd-smooth-update2}
\end{gather}
where the norm is no longer the Euclidean norm, but the norm in
question. To evaluate the minimum on the right side, we can use basic
Fenchel duality: given a function $g$, its Fenchel dual is defined as
\[ g^\star(\theta) := \max_z \{ \ip{\theta, z} - g(z) \}. \] If we
define $g(z) = \frac12 \norm{z}^2$, it is known that
$g^\star(\theta) = \frac12 \norm{z}_\star^2$ (see~\cite[Example~3.27]{BV}). Hence
\begin{align}
 \min_x \Big\{ \frac{1}{2} \norm{ x - x_t }^2 + \eta_t \ip{ \gr_t, x - x_t
  } \Big\} &= - \max_x \Big\{ \eta_t \ip{ \gr_t, x_t - x
             } - \frac{1}{2} \norm{ x_t - x }^2 \Big\} \notag \\
  &= - \max_z \Big\{ \ip{ \eta_t \gr_t, z
             } - \frac{1}{2} \norm{ z }^2 \Big\} = - \frac{1}{2} \norm{
    \eta_t \gr_t }_\star^2 \label{eq:24}
\end{align}
If a function $f$ is $\beta$-smooth with respect to the norm, then
setting $\eta_t = \nicefrac1\beta$ gives:
\begin{align*}
  f(x_{t+1}) &\stackrel{\text{~(\ref{eq:smooth})}}{\leq} f(x_t) + \ip{ \gr_t, x_{t+1} - x_t} + \frac\beta2
               \norm{x_{t+1} - x_t}^2 \\
             &= f(x_t) + \beta \Big( \ip{ \eta_t \gr_t, x_{t+1} - x_t} + \frac12
               \norm{x_{t+1} - x_t}^2 \Big)
               = f(x_t) + \beta \cdot \Big( - \frac{1}{2} \norm{
               \eta_t \gr_t }_\star^2\Big),
\end{align*}
where the last equality uses that $x_{t+1}$ is the minimizer of the
expression in~(\ref{eq:24}). Summarizing, we get
\begin{gather}
  f(x_{t+1}) \leq f(x_t) - \frac{1}{2\beta} \norm{ \gr_t }_\star^2. \label{eq:sm-step2}
\end{gather}
This is analogous to the expression~(\ref{eq:gd-smooth-update}). Now we
can continue the proof as in \S\ref{sec:eucl-smooth}, again defining
$D := \max\{ \|x - x^*\| \mid f(x) \leq f(x_0)\}$, and using the
generalized Cauchy-Schwarz inequality to get the general-norm analog of
Theorem~\ref{thm:smoothII}:
\begin{theorem}[GD: Smooth Functions for General Norms]
  \label{thm:smooth-GD}
  Given a function $f$ that is $\beta$-smooth with respect to the norm
  $\norm{\cdot}$, the update rule~(\ref{eq:gd-smooth-update2}) ensures
  \begin{gather}
    f(x_T) - f(x^*) \leq \beta \; \frac{2D^2}{T+1}. \notag
  \end{gather}
\end{theorem}

\newpage

\section{Nesterov Acceleration: A Potential Function Proof}
\label{sec:nesterov}

In \S\ref{sec:smooth}, we proved a convergence rate of $O(1/T)$
for smooth functions, using both projected gradient descent and the
Frank-Wolfe method. But the lower bound is only $\Omega(1/T^2)$. In this
case, the algorithm can be improved: Yurii Nesterov showed how to do it
using his \emph{accelerated gradient descent} methods \cite{Nest83}. Recently there
has been much interest in gaining a deeper understanding of this
process, with proofs using ``momentum'' methods and continuous-time
updates~\cite{Wibi,SBC,KBB15,WilsonRJ16,DO17}.

Let us now see potential-based proofs for his theorem, both for the
smooth case, and for the well-conditioned case.  We consider only the
unconstrained case (i.e., when $K = \R^n$) and the Euclidean norm; the
extension to general norms is sketched in \S\ref{sec:nesterov-norms}.

\subsection{An Illustrative Failed Attempt}
\label{sec:nest-failed}

One way to motivate Nesterov's accelerated algorithm is to revisit the
proof for smooth functions in \S\ref{sec:takeIII}. Let us recall
the essential facts.  The potential was
\[ \boxed{ \Phi_t = t\, (f(x_t) - f(x^*)) + \frac{a}{2} \|x_t - x^*\|^2
  } \] for some $a > 0$. Hence the potential difference was:
\begin{align*}
  \Phi_{t+1} - \Phi_t & =   (t+1) \underbrace{(f(x_{t+1}) -
                        f(x_t))}_{(\ref{eq:sm-step})} + \underbrace{f(x_t) -
                        f(x^*)}_{\text{(convexity)}} + \frac{a}{2}\, \underbrace{(\norm{x_{t+1} - x^*}^2 -
                        \norm{x_t - x^*}^2
                        )}_{(\ref{eq:potchange-basic})}. \label{eq:18} \\
  \leq\ & (t+1) \cdot \overbrace{\ts -\frac{1}{2\beta} \norm{\gr_t}_2^2} +
          \overbrace{\ip{\gr_t,x_t - x^*}}  + a \big( \overbrace{\ts
          \eta_t \ip{\gr_t, x^* -
          x_t } + \frac{\eta_t^2}{2} \norm{\gr_t}_2^2} \big) = -\nicefrac{t}{2\beta} \norm{\gr_t}^2
          \leq 0. 
\end{align*}
In that last expression we set $\eta_t = \nicefrac{1}{\beta}$ and
$a = 1/\eta_t = \beta$ to cancel the inner-product terms.

Observe that the potential may be decreasing considerably, by $-\nicefrac{t}{2\beta} \norm{\gr_t}^2$, but we
are ignoring this large decrease. If we want to a show an $O(1/t^2)$ rate of convergence, a first (incorrect) attempt to get a
better analysis would be to try to apply the analysis above with the potential changed to
\[ \Phi_t = \balert{t(t+1)}\, (f(x_t) - f(x^*)) + \frac{a}{2} \|x_t - x^*\|^2. \]
In particular note the factor $a_t=t(t+1)$ instead of $(t+1)$ above. 

At first glance, the potential change $\Phi_{t+1}- \Phi_t$ would be
\begin{gather}
   \balert{(t+2)(t+1)} \cdot \ts \underbrace{(-\frac{1}{2\beta}
  \norm{\gr_t}_2^2)}_{\text{(\ref{eq:sm-step})}}) + \balert{2(t+1)} \cdot \ip{\gr_t,x_t - x^*} + a \big( \ts
  \eta_t \ip{\gr_t, x^* - x_t } + \frac{\eta_t^2}{2} \norm{\gr_t}_2^2
  \big) \label{eq:25}
\end{gather}
Now if we change the step length $\eta_t$ from $1/\beta$ to something
more aggressive, say $\eta_t = \frac{t+1}{2\beta}$, and choose $a =
4\beta$, the inner-product terms cancel, and the potential reduction
seems to be at most
\[ \norm{\gr_t}^2 \bigg( - \frac{(t+1)(t+2)}{2\beta} + \frac{(t+1)^2}{2\beta} \bigg)
\leq 0. \] This would seem to give us an $O(1/T^2)$ convergence, with
the standard update rule.

So where's the mistake? It's in our erroneous use of~(\ref{eq:sm-step})
in~(\ref{eq:25}) --- we used the cautious update with $\eta_t =
\nicefrac1\beta$ to get the first term of $-\frac{1}{2\beta}
\norm{\gr_t}_2^2$, but the aggressive update with $\eta_t =
\frac{t+1}{2\beta}$ elsewhere. To fix this, how about running \emph{two}
processes, one cautious and one aggressive, and then combining them
together linearly (with decreasing weight on the aggressive term) to get
the new point $x_t$? This is precisely what Nesterov's Accelerated
Method does; let's see it now. (This is another way to arrive at the
elegant linear-coupling view that Allen-Zhu and Orecchia present
in~\cite{AO17}.)

\subsection{Getting the Acceleration}
\label{sec:nest-proof}

As we just sketched, one way to view the accelerated gradient descent
method is to run two iterative processes for $y_t$ and $z_t$, and then
combine them together to get the actual point $x_t$. The proof is almost
the same as before.

\textbf{The Update Steps:} Start with $x_0= y_0 = z_0$. At time
$t$, play $x_t$. For brevity, define $\gr_t := \gr f(x_t)$. Now consider
the update rules, where the color is used to emphasize the subtle
differences (in particular, $z$ is updated by the gradient at $x$ in \eqref{eq:3}):
\begin{empheq}[box=\widefbox]{gather}
  \textstyle y_{t+1} \gets \balert{x_t} - \frac{1}{\beta} \gr f(\balert{x_t}) \label{eq:4} \\
  z_{t+1} \gets \balert{z_t} - \eta_t \gr f(\balert{x_t}) \label{eq:3} \\
  x_{t+1} \gets (1 - \tau_{t+1}) y_{t+1} + \tau_{t+1} z_{t+1}. \label{eq:5}
\end{empheq}
In~(\ref{eq:3}), we will choose the ``aggressive'' step size
$\eta_t = \smash{\frac{t+1}{2\beta}}$ as we did in the above failed
attempt. In~(\ref{eq:5}) the mixing weight is
$\tau_t = \smash{\frac{2}{t+2}}$, but this will
arise organically below. 

\textbf{The Potential:} This is the same  one from the failed attempt:
\begin{gather}
  \boxed{ \Phi(t) = t(t+1) \cdot ( f(\balert{y_t} ) - f(x^*)) + 2\beta \cdot \| \balert{z_t}
  - x^*\|^2} \label{eq:8}
\end{gather}

\medskip
\textbf{The Potential Change:}
Define $\Delta \Phi_t = \Phi(t+1) - \Phi(t)$. By the standard GD
analysis in~(\ref{eq:potchange-basic}), 
\begin{gather}
 \frac12  (\| z_{t+1} - x^*\|^2 - \| z_{t} - x^*\|^2) =
 \frac{\eta_t^2}{2} \| \gr_t \|^2  +  \eta_t \ip{ \gr_t, x^* -z_t }
 \label{eq:nest-gd} \\
\implies
  \Delta \Phi_t = t(t+1) \cdot( f(\balert{y_{t+1}}) - f(y_t) ) +
  2(t+1) \cdot (f(\balert{y_{t+1}}) - f(x^*)) + 4\beta\bigg( \frac{\eta_t^2}{2} \| \gr_t
  \|^2 + \eta_t \ip{ \gr_t, x^* - z_t } \bigg) \notag 
\end{gather}
By smoothness and the update rule for $y_{t+1}$, (\ref{eq:sm-step}) implies
$f(\balert{y_{t+1}}) \leq f(\balert{x_t}) - \frac{1}{2\beta}
\norm{\gr_t}^2$. Substituting, and dropping the resulting (negative) squared-norm term,
\begin{align*}
  \Delta \Phi_t &\leq t(t+1) \cdot( f(x_t) - f(y_t) ) + 2(t+1) \cdot
  (f(x_t) - f(x^*)) + 4\beta
  \eta_t \ip{ \gr_t, x^* - z_t } \notag \\
  &\leq t(t+1) \cdot \ip{\gr_t, x_t - y_t} + 2(t+1) \cdot
  \ip{\gr_t,x_t - x^*} + 2(t+1)\cdot \ip{ \gr_t, x^* - z_t } \notag
\end{align*}
using convexity for the first two terms, and  $\eta_t =
\frac{t+1}{2\beta}$ for the last one. Collecting like terms, 
\begin{gather}
  \Delta \Phi_t \leq (t+1) \cdot \ip{\gr_t, (t+2)x_t - ty_t - 2z_t} = 0,
\end{gather}
by using~(\ref{eq:5}) and $\tau_t = \frac{2}{t+2}$.
Hence $ \Phi_{t} \leq \Phi_{0}$ for all $t \geq 0$. This proves:

\begin{theorem}[Accelerated GD]
  \label{thm:nest-agm2}
  Given a $\beta$-smooth function $f$, the update rules~(\ref{eq:4})-(\ref{eq:5})
  ensure
  \begin{gather}
    f(y_t) - f(x^*) \leq 2 \beta \; \frac{\norm{z_0 - x^*}^2}{t(t+1)}. \notag
  \end{gather}
\end{theorem}

\subsubsection{An Aside: Optimizing Parameters and Other Connections}

Suppose we choose the generic potential
$\Phi(t) = \lambda_{t-1}^2 (f(y_t) - f(x^*)) + \frac{\beta}{2} \norm{z_t
  - x^*}^2$, where $\lambda_t = \Theta(t^2)$, and try to optimize the
calculation above. Having
$\lambda_t^2 - \lambda_{t-1}^2 = \lambda_t$ and
$\tau_t = 1/\lambda_t$ makes the calculations work out very
cleanly. Solving this recurrence leads to the (somewhat exotic-looking) weights
\begin{gather}
  \lambda_0=0,~~~~\lambda_t = \frac{1 + \sqrt{1 + 4\lambda_{t-1}^2}}{2} \label{eq:23}
\end{gather}
used in the
standard exposition of \texttt{AGM2}.

The update rules~(\ref{eq:4})--(\ref{eq:5}) have sometimes been called
\texttt{AGM2} (Nesterov's second accelerated method) in the
literature. A different set of update rules (called \texttt{AGM1}) are
the following: for the optimized choice of $\lambda_t$
from~(\ref{eq:23}), define:
\begin{empheq}[box=\widefbox]{align}
  y_{t+1} &\gets x_t - \frac{1}{\beta} \gr f(x_t) \label{eq:9} \\
  x_{t+1} &\gets \bigg(1- \frac{1-\lambda_t}{\lambda_{t+1}}\bigg) y_{t+1} +
  \frac{1-\lambda_t}{\lambda_{t+1}} y_t \label{eq:10}
\end{empheq}

Let us show the simple equivalence (also found in, e.g.,~\cite{Tseng,DT14,
  KimF16}).

\begin{lemma}
  \label{lem:equiv-FGM2-AZO}
  Using updates (\ref{eq:9}--\ref{eq:10}) and setting
  $z_t := \lambda_t x_t - (\lambda_t -1)y_t = \lambda_{t} (x_{t} -
  y_{t}) + y_t $, and $\tau_t := 1/\lambda_t$ leads to the
  updates~(\ref{eq:4}--\ref{eq:5}).
\end{lemma}
\begin{proof}
  Clearly $y_t$ is the same as above, so it suffices to show that $z_t$
  and $x_t$ behave identically. Indeed, rewriting the definition of
  $z_t$ and substituting $\tau_t = 1/\lambda_t$ gives
  \begin{align}
    x_t &= (1-\tau_t) y_t + \tau_t z_t. \notag \\
    z_{t+1} - z_{t} &= \left( \lambda_{t+1} (x_{t+1} - y_{t+1}) + y_{t+1}
                      \right) - \left( \lambda_{t} (x_{t} - y_{t}) +
                      y_{t} \right)\label{eq:11} 
  \end{align}
  Moreover, rewriting~(\ref{eq:10}) gives
  \begin{gather}
    \lambda_{t+1} (x_{t+1} - y_{t+1} ) - (1-\lambda_t) (y_t - y_{t+1})
    = 0 \label{eq:12}
  \end{gather}
  Subtracting~(\ref{eq:12}) from (\ref{eq:11}) gives
  \begin{gather}
    z_{t+1} - z_{t} = \lambda_t y_{t+1} - \lambda_{t} x_t =
    -\frac{\lambda_t}{\beta} \gr f(x_t)
  \end{gather}
  Recalling that $\lambda_t = 1/\tau_t = (\beta \eta_t)$, this is
  precisely the update rule $z_{t+1} \gets z_t - \eta_t \gr
  f(x_t)$. This shows the equivalence of the two update rules.
\end{proof}

\subsection{The Constrained Case with Acceleration}

\textbf{The Update Steps:} The update rule is very similar to the one
above, it just involves projecting the points onto the body
$K$. Formally, again we start with $x_0= y_0 = z_0$. At time $t$, play
$x_t$. For brevity, define $\gr_t := \gr f(x_t)$. Now consider the
update rules, where again the color is used to emphasize the subtle
differences:
\begin{empheq}[box=\widefbox]{gather}
  \textstyle y_{t+1} \gets \Pi_K (\balert{x_t} - \frac{1}{\beta} \gr f(\balert{x_t})) \label{eq:4c} \\
  z_{t+1} \gets \Pi_K(\balert{z_t} - \eta_t \gr f(\balert{x_t})) \label{eq:3c} \\
  x_{t+1} \gets (1 - \tau_{t+1}) y_{t+1} + \tau_{t+1} z_{t+1}. \label{eq:5c}
\end{empheq}

We now show that this update rule satisfies same guarantee as in Theorem \ref{thm:nest-agm2} for the unconstrained case.
\begin{theorem}[Accelerated GD]
  \label{thm:nest-agm2-cond}
  Given a $\beta$-smooth function $f$, the update rules~(\ref{eq:4c})-(\ref{eq:5c})
  ensure
  \begin{gather}
    f(y_t) - f(x^*) \leq 2 \beta \; \frac{\norm{z_0 - x^*}^2}{t(t+1)}. \notag
  \end{gather}
\end{theorem}


\textbf{The Potential:} 
\begin{gather}
  \boxed{ \Phi(t) = t(t+1) \cdot ( f(\balert{y_t} ) - f(x^*)) + 2\beta \cdot \| \balert{z_t}
  - x^*\|^2} \label{eq:8c}
\end{gather}

\textbf{Change in Potential:} 
\begin{align*}
 \Phi_{t+1} - \Phi_t &= \underbrace{(t+1)(t+2) ( f(y_{t+1} ) - f(x_t) )}
                       \;\; \underbrace{ - t(t+1) (
  f(y_{t} ) - f(x_t) ) + 2(t+1) (f(x_t) - f(x^*))} \\ &+ \underbrace{2\beta \cdot
  \left( \| z_{t+1}
  - x^*\|^2 - \| z_t
  - x^*\|^2 \right)}
\end{align*}

Using convexity on both differences in the second group gives 
  \begin{align*}
             &\leq - t(t+1) \ip{ \gr_t, y_t - x_t } + 2(t+1) \ip{ \gr_t, x_t - x^* } \\ &= (t+1) \cdot \ip{ \gr_t, - t(y_t -
                                      x_t) + 2(x_t - x^*) }.
  \end{align*}
  Now $(1-\tau_t)(y_t - x_t) = \tau_t (x_t - z_t)$, and  
  if $\tau_t = \frac{2}{t+2}$, then we get $t(y_t - x_t) = 2(x_t -
  z_t)$, and hence the above expression is
  \[ = 2(t+1) \ip{  \gr_t, z_t - x^*
    }. \]
	 The third term is
		\[ 2\beta (  \ip { z_{t+1} - z_{t},z_{t+1}-x^*} - \norm{z_{t+1}-z_t}^2.\] 
By the Pythagorean property, 
	 $\ip{z'_{t+1}-z_{t+1},z_{t+1} - x^*} \geq 0$ where $z'_{t+1} = z_t - \eta_t \gr_t$. 		
Setting $\eta_t  = (t+1)/\beta$ and multiplying by $2\beta$, adding to the term above, the third term is upper bounded by 
\[ -2(t+1) \ip{\gr_t,z_{t+1}-x^*}\]
Combine this with the second term above to cancel $x^*$, it suffices to show the following claim:
\begin{lemma}
  \[ \overbrace{(t+1)(t+2) ( f(y_{t+1} ) - f(x_t) )}  + \overbrace{ 2(t+1) \ip{  \gr_t,  z_t -z_{t+1}} -  \norm{z_{t+1}-z_t}^2 } \leq 0. \]
\end{lemma}

\begin{proof}
  By smoothness,
  \[ f(y_{t+1} ) - f(x_t) \leq \ip{ \gr_t, y_{t+1} - x_t } +
    \frac{\beta}{2} \| y_{t+1} - x_t \|^2. \]
As  $\min_{y \in K} \{ \ip{ \gr_t, y - x_t } +
  \frac{\beta}{2} \| y - x_t \|^2 \}$ in minimized by the definition of $y_{t+1}$, we get that
for any $v\in K$,	the RHS is
  \[ \leq \ip{ \gr_t, v - x_t } + \frac{\beta}{2} \| v - x_t \|^2. \]
  Define $v := (1 - \tau_t)y_t + \tau_t z_{t+1} \in K$, so
  $v - x_t = \tau_t(z_{t+1} - z_t)$. Substituting, we get
  \begin{align*}
   f(y_{t+1} ) - f(x_t) &\leq \tau_t \ip{ \gr_t, z_{t+1} - z_t } +
                          \frac{\tau_t^2 \beta}{2} \| z_{t+1} - z_t \|^2
  \end{align*}
  Setting $\tau_t = \frac{2}{t+2}$, the claim follows.
\end{proof}

\subsection{The Extension to Arbitrary Norms}
\label{sec:nesterov-norms}

Given an arbitrary norm $\norm{\cdot}$, the update rules now use the
gradient descent update~(\ref{eq:gd-smooth-update2}) for smooth
functions for the $y$ variables, and the mirror descent update rules
(\ref{eq:md}) for the $z$ variables:
\begin{empheq}[box=\widefbox]{align}
  y_{t+1} &\gets \arg\min_{y} \bigg\{ \frac{\beta}{2} \norm{ y -
      x_t }^2 + \ip{ \gr f(x_t), y - x_t }
    \bigg\} \notag \\
    z_{t+1} &\gets \arg\min_z \bigg\{ \ip{  \eta_t \gr f(x_t), z} +
    \hdiv{z}{z_t} \bigg\} \label{eq:nest-gennorm}\\ 
  x_{t+1} &\gets (1 - \tau_{t+1}) y_{t+1} + \tau_{t+1} z_{t+1}. \notag  
\end{empheq}
Given the discussion in the preceding sections, the update rules are the
natural ones: the first is the update~(\ref{eq:gd-smooth-update2}) for
smooth functions, and the second is the usual mirror descent update
rule~(\ref{eq:md-prox}) for the strongly-convex function $h$. The step
size is now set to $\eta_t = \frac{(t+1) \alpha_h}{2\beta}$, and the
potential function becomes:
\begin{gather}
  \boxed{ \Phi(t) = t(t+1)\cdot ( f(y_t ) - f(x^*)) +
    \frac{4\beta}{\alpha_h} \cdot \hdiv{x^*}{z_t} }, \label{eq:8md}
\end{gather}
which on substituting $\hdiv{x^*}{z_t} = \frac12 \norm{ x^* - z_t }^2$
and $\alpha_h = 1$ gives~(\ref{eq:8}).

We already have all the pieces to bound the change in potential. Use the
mirror descent analysis~(\ref{eq:pot-change-md}) to get
\[ \hdiv{x^*}{z_{t+1}} - \hdiv{x^*}{z_{t}} \leq
  \frac{\eta_t^2}{2\alpha_h}\; \norm{ \gr_t }_*^2 + \eta_t \langle
  \gr_t, x^*-x_{t}\rangle.
\]
which replaces~(\ref{eq:nest-gd}). Infer $f(y_{t+1}) \leq f(x_t) -
\frac{1}{2\beta} \norm{\gr_t}_*^2$ from~(\ref{eq:sm-step2}) in the
smooth case. Substitute these into the analysis from
\S\ref{sec:nest-proof} (with minor changes for the $\alpha_h$ term) to
get the following theorem:

\begin{theorem}[Accelerated GD: General Norms]
  \label{thm:nest-agm2-norms}
  Given a $\beta$-smooth function $f$ with respect to norm
  $\norm{\cdot}$, the update rules~(\ref{eq:nest-gennorm}) ensure
  \begin{gather*}
    f(y_t) - f(x^*) \leq \frac{4\beta}{\alpha_h}\cdot \frac{ \hdiv{x^*}{z_0} - \hdiv{x^*}{z_t}}{t(t+1)}.
  \end{gather*}
\end{theorem}

\subsection{Strongly Convex with Acceleration}

Now consider the case when the function $f$ is well-conditioned with
condition number $\kappa = \beta/\alpha$, and describe the algorithm of Nesterov
with convergence rate $\exp(-t/\sqrt{\kappa})$ \cite{Nest-book}. Again, this 
is the best possible for first order methods.

Before describing the algorithm and its analysis, it is instructive to see how 
the result for the smoothed case from section \ref{sec:nest-proof} already gives (in principle)
such a result for the well-conditioned case. 

\medskip\textbf{Reduction from Smooth Case:} 
Theorems \ref{thm:nest-agm2} and \ref{thm:nest-agm2-cond} for the (constrained) smoothed case 
give that
\[ f(y_t) - f(x^*) \leq  2 \beta \frac{\|x_0 - x^*|^2}{t(t+1)}\]
Together with strong convexity,
\[f(y_t) -f(x^*) \geq \frac{\alpha}{2} \|y_t - x^*\|^2, \]
this gives  $\|y_t - x^*\|^2 \leq 4 \kappa  \|x_0-x^*\|^2/t(t+1)$. 

So in $t = 4 \sqrt{\kappa}$ steps, the distance $\|y_t-x^*\|$ is at most half 
the initial distance $\|x_0-x^*\|$ from the optimum.
Starting the algorithm again with $y_t$ as the initial point, and iterating this 
process thus gives an overall algorithm, that in $t$ steps has error at most 
$2^{-t/4\sqrt{\kappa}} \|x-x_0\|$.

Restarting the algorithm after every few steps is not ideal, and 
we now describe Nesterov's algorithm with this improved convergence rate.  
For simplicity
only consider the unconstrained case.

\medskip\textbf{The Update Rules:} 
 We now use the following
updates (which look very much like the~\texttt{AGM1} updates):
\begin{empheq}[box=\widefbox]{align}
  y_{t+1} &\gets x_t - \frac{1}{\beta} \gr f(x_t) \label{eq:13} \\
  x_{t+1} &\gets \bigg(1 + \frac{\sqrt{\kappa} -1 }{\sqrt{\kappa} + 1}\bigg) y_{t+1} -
  \frac{\sqrt{\kappa} -1 }{\sqrt{\kappa} + 1} y_t. \label{eq:14}
\end{empheq}
For the analysis, it will be convenient to define
$\tau = \frac{1}{\sqrt{\kappa} +1}$ and set
\begin{gather}
 \textstyle z_{t+1} := \frac{1}{\tau}\; x_{t+1}  - \frac{1-\tau}{\tau} \; y_{t+1}. \label{eq:15}
\end{gather}
We now show that that the error after $t$ steps is
\begin{gather}
  f(y_t)- f(x^*) \leq (1+\gamma)^{-t} \left(\frac{\alpha+ \beta}{2}\;
    \|x_0-x^*\|^2 \right), \label{eq:nest-wcond}
\end{gather}
where $\gamma = \frac{1}{\sqrt{\kappa}-1}$ (as in \S\ref{sec:well-cond}, for $\kappa=1$ the algorithm reaches optimum in a single step and $y_1=x^*$, and hence we assume that $\kappa>1$). This improves on the error
of $(1 + 1/\kappa)^{-t} \frac\beta2 \norm{x_0 - x^*}^2$ we get from
\S\ref{sec:well-cond}.

\medskip\textbf{The Potential:} 
Consider the potential
\[ \boxed{ \Phi(t) = (1+\gamma)^{t} \left( f(y_t) - f(x^*) +
      \frac{\alpha}{2} \|z_t - x^*\|^2 \right).}  \] Observe that
$\Phi_0 = f(y_0) -f(x^*) + \frac{\alpha}{2} \|z_0 - x^*\|^2$.  As
$x_0=y_0=z_0$, and by $\beta$-smoothness of $f$, 
\[ \Phi_0 \leq \frac{\alpha+ \beta}{2} \; \|x_0-x^*\|^2. \] 

\medskip\textbf{Change in Potential:} To
show the error bound~(\ref{eq:nest-wcond}), it suffices to show
that $\Delta \Phi(t) = \Phi(t+1)-\Phi(t) \leq 0$ for each $t$. This is
equivalent to showing
\[ (1+\gamma) ( f(y_{t+1}) - f(x^*)) - ( f(y_t) - f(x^*)) +
  \frac{\alpha}{2} \left( (1+\gamma) \|z_{t+1} - x^*\|^2 -\|z_t -
    x^*\|^2 \right) \leq 0 \] We first bound the terms involving $f$ in
the most obvious way. As above, we use $\gr_t$ as short-hand for
$\gr f(x_t)$. By $\beta$-smoothness and the update rule,
again $ f(y_{t+1}) \leq f(x_t) - \frac{1}{2\beta}\norm{\gr_t}^2$.
So,
\begin{align}
  (1+\gamma) ( f(y_{t+1}) 
  & - f(x^*)) -  ( f(y_t) - f(x^*)) \notag \\
  & \leq   f(x_t) - f(y_t)  + \gamma (f(x_t) - f(x^*))  - (1+\gamma)
    \frac1{2\beta}\norm{\gr_t}^2 \notag \\
  & \leq \ip{\gr_t,x_t-y_t}  + \gamma \bigg( \ip{\gr_t,x_t-x^*} -
    \frac{\alpha}{2} \|x_t - x^*\|^2 \bigg) - 
    \frac{1+\gamma}{2\beta}\norm{\gr_t}^2,  \label{st1} 
\end{align}
where the last inequality used convexity and strong convexity
respectively. 

We now want to remove references to $y_t$. By definition~(\ref{eq:15}),
$z_t = (\frac1\tau - 1)(x_t - y_t) + x_t = \sqrt{\kappa} (x_t - y_t) +
x_t$, so we infer
$\gamma(z_t - x^*) = \sqrt{\kappa}\gamma (x_t - y_t) + \gamma (x_t -
x^*)$. Using $\sqrt{\kappa}\gamma = 1+\gamma$, simple algebra gives
$(x_t - y_t) + \gamma (x_t - x^*) = \frac{1}{1+\gamma} \big(\gamma(z_t -
x^*) + \gamma^2 (x_t - x^*) \big) $.

For brevity we use $X_t := x_t - x^*$, $Z_t = z_t - x^*$, and substitute
the above expression into~(\ref{st1}) to get 
\begin{gather}
  \frac{1}{1+\gamma} \ip{\gr_t, \gamma Z_t + \gamma^2 X_t} -
  \frac{\alpha\gamma}{2} \|X_t\|^2  -
  \frac{1+\gamma}{2\beta}\norm{\gr_t}^2. \label{st1b} 
\end{gather}
Now, let us upper bound the terms in $\Delta \Phi(t)$ involving $z$. Conveniently, we can relate
$z_{t+1}$ and $z_t$ using a simple calculation that we defer for the moment.
\begin{claim}
  \label{clm:mystery}
$z_{t+1} = (1-\frac{1}{\sqrt{\kappa}}) z_t + \frac{1}{\sqrt{\kappa}}
x_t - \frac{1}{\alpha \sqrt{\kappa}} \gr_t$ and so 
$z_{t+1} - x^* = \frac{1}{1+\gamma} Z_t + \frac{\gamma}{1+\gamma} X_t  - \frac{\gamma}{\alpha(1+\gamma)} \gr_t$.
\end{claim}

Now use Claim~\ref{clm:mystery} and expand using
$\norm{a+b+c}^2 = \norm{a}^2 + \norm{b}^2 + \norm{c}^2 + 2\ip{a,b} +
2\ip{b,c} + 2\ip{a,c}$:
\begin{align}
  &(1+\gamma)\norm{z_{t+1} - x^*}^2 - \norm{z_t - x^*}^2 \notag\\
&= \frac{1}{1+\gamma} \left( \norm{Z_t}^2 + \gamma^2\norm{X_t}^2 +
  \frac{\gamma^2}{\alpha^2}\norm{\gr_t}^2 +
  2{\gamma}\ip{Z_t, X_t} -
  \frac{2\gamma}{\alpha}\ip{ \gr_t, Z_t} -
  \frac{2\gamma^2}{\alpha}\ip{ \gr_t, X_t} \right) - \norm{Z_t}^2. \label{eq:st2b}
\end{align}
Now sum~(\ref{st1b}) and $\alpha/2$ times~(\ref{eq:st2b}). The terms
involving $\norm{\gr_t}^2$ cancel since $\frac{1+\gamma}{2\beta} =
\frac{\alpha\gamma^2}{2\alpha^2(1+\gamma)}$ (by the definition of $\gamma$). Moreover, the inner-product
terms involving $\gr_t$ also cancel. Hence the potential change is at
most
\begin{align}
  \Delta \Phi(t) &\leq \frac{\alpha\gamma}{2} \norm{X_t}^2 \bigg(-1 +
  \frac{\gamma}{1+\gamma}\bigg) + \frac{\alpha}{2} \norm{Z_t}^2 \bigg(
  \frac{1}{1+\gamma} - 1 \bigg) 
   +
  \frac{\alpha\gamma}{1+\gamma} \ip{Z_t, X_t} \notag \\
  &= - \frac{\alpha\gamma}{2(1+\gamma)} \left( \norm{X_t}^2 + \norm{Z_t}^2 -
    2\ip{ Z_t, X_t } \right) \\
  &= - \frac{\alpha\gamma}{2(1+\gamma)} \norm{Z_t -
      X_t}^2 \leq 0.
\end{align}
Hence the potential does not increase, as claimed. It only remains to
prove Claim~\ref{clm:mystery}.

\begin{proof}[Proof of Claim~\ref{clm:mystery}]
 The expression of $x_{t+1}$ from~(\ref{eq:14}) can be written as $(2-2\tau) y_{t+1} - (1-2\tau) y_t$. Plugging into the expression for
  $z_{t+1}$ from~(\ref{eq:15}) gives
  \begin{align*}
    z_{t+1} &= \frac{1}{\tau} \left( (2-2\tau) y_{t+1} - (1-2\tau)  y_t - (1-\tau)
              y_{t+1} \right) \\ &= 
							\frac{1}{\tau} \left( (1-\tau) y_{t+1} - (1-2\tau)  y_t \right).\\
    \intertext{Using the update rule~(\ref{eq:13}) for $y_{t+1}$, and the relation $x_t = (1-\tau)y_t + \tau
    z_t$ to eliminate $y_t$}
            &= \frac{1}{\tau} \left( (1-\tau) \bigg(x_t - \frac{1}{\beta} \gr_t\bigg) - \frac{ (1-2\tau)}{1-\tau} (x_t - \tau z_t) \right) \\
            &= \frac{1-2\tau}{1-\tau} z_t  + \frac{\tau}{1-\tau} x_t   - \frac{1-\tau}{\tau\beta} \gr_t.
  \end{align*}
  Using $\tau = 1/(\sqrt{\kappa}+1)$ and $\beta = \kappa \alpha$ now gives  the claim.
\end{proof}

\subsection*{Acknowledgments} 

We thank S\'ebastien Bubeck, Daniel Dadush, Jelena Diakonikolas, Nick Harvey, Elad
Hazan, Greg Koumoutsos, Raghu Meka, Aryan Mokhtari, Marco Molinaro,
Thomas Rothvo\ss, and Kunal Talwar for their helpful comments. We also
thank the referees for their detailed comments that helped improve the
presentation substantially.

{\small 
\bibliography{GD-notes}
\bibliographystyle{amsalpha}
}

\end{document}